%% file: main.tex
\documentclass{article}

\PassOptionsToPackage{numbers}{natbib}
\usepackage[preprint]{neurips_2026}

\usepackage[utf8]{inputenc}
\usepackage[T1]{fontenc}

\usepackage{amsmath}
\usepackage{amssymb}
\usepackage{amsfonts}
\usepackage{mathtools}
\usepackage{amsthm}
\usepackage{wrapfig}
\usepackage{caption}

\usepackage{graphicx}
\usepackage{flafter}
\usepackage{booktabs}
\usepackage{multirow}
\usepackage{subcaption}

\usepackage{tikz}
\usepackage{pgfplots}
\usepgfplotslibrary{groupplots,fillbetween}
\pgfplotsset{compat=1.18}

\usepackage{algorithm}
\usepackage{algorithmic}

\usepackage{nicefrac}
\usepackage{microtype}
\usepackage{xcolor}
\usepackage{url}

\usepackage{hyperref}
\usepackage[capitalize,noabbrev]{cleveref}
\usepackage{amsthm}

\theoremstyle{plain}
\newtheorem{theorem}{Theorem}[section]
\newtheorem{lemma}[theorem]{Lemma}
\newtheorem{corollary}[theorem]{Corollary}

\theoremstyle{definition}
\newtheorem{definition}[theorem]{Definition}

\theoremstyle{remark}
\newtheorem{remark}[theorem]{Remark}
\usepackage{amsthm}

\title{Learning Domain Trajectories with Flow Matching for Gradual Domain Adaptation}

\author{%
  \textbf{Yubo Huang}$^{1,}$\thanks{Yubo Huang and Yushe Cao contributed equally to this work.}\quad
  \textbf{Yushe Cao}$^{2,}$\footnotemark[1]\quad
  \textbf{Zixi Wang}$^{3}$\\
  \textbf{Jingzehua Xu}$^{4,}$\thanks{Co-corresponding author: Jingzehua Xu (\href{mailto:u3014074@connect.hku.hk}{\nolinkurl{u3014074@connect.hku.hk}}).}\quad
  \textbf{Xin Lai}$^{1}$\quad
  \textbf{Xiongyu Hu}$^{1,}$\thanks{Primary corresponding author: Xiongyu Hu (\href{mailto:huxiongyu@swjtu.edu.cn}{\nolinkurl{huxiongyu@swjtu.edu.cn}}).}\quad
  \textbf{Shuai Zhang}$^{5}$\\[4pt]
  \parbox{0.96\textwidth}{\centering\normalfont\small
    $^{1}$Key Laboratory of Transportation Tunnel Engineering, Ministry of Education,\\
    Southwest Jiaotong University, Chengdu, China\\
    $^{2}$Tsinghua University, Beijing, China\\
    $^{3}$School of Information and Software Engineering,\\
    University of Electronic Science and Technology of China, Chengdu, China\\
    $^{4}$Department of Mechanical Engineering, The University of Hong Kong, Hong Kong, China\\
    $^{5}$Department of Data Science, New Jersey Institute of Technology, Newark, NJ, USA
  }%
}
\hypersetup{
  hidelinks,
  pdftitle={Learning Domain Trajectories with Flow Matching for Gradual Domain Adaptation},
  pdfauthor={Yubo Huang, Yushe Cao, Zixi Wang, Jingzehua Xu, Xin Lai, Xiongyu Hu, Shuai Zhang}
}

\begin{document}

\maketitle

\begin{abstract}
\input{sections/00_abstract}
\end{abstract}

\input{sections/01_introduction}
\input{sections/02_related_work}
\input{sections/03_method}

\input{sections/04_experiments}
\input{sections/05_conclusion}

\renewcommand{\acksection}{\section*{Acknowledgments}}
\begin{ack}
\input{sections/acknowledgments}
\end{ack}

\section*{Competing Interests}
\input{sections/competing_interests}

\bibliographystyle{plain}
\bibliography{references}

\appendix
\input{sections/06_limitations}
\input{sections/appendix/appendix}

\end{document}

%% file: sections/00_abstract.tex
Large distribution shifts remain a primary obstacle for domain generalization and unsupervised adaptation. Although gradual domain adaptation alleviates this difficulty by introducing ordered intermediate domains, most existing approaches still operate in a discrete, step-wise manner, which can accumulate errors and often relies on unstable Min--Max optimization. We propose \textbf{FlowGDA}, a flow-matching framework that models gradual adaptation as \emph{continuous} transport on a statistical manifold in a representation space. FlowGDA learns time-dependent velocity fields that induce smooth domain trajectories, using a simulation-free regression objective rather than adversarial training. To prevent semantic drift, we couple left-to-right transport between consecutive domains with a source-anchored transport and a semantic consistency loss, and we further stabilize learning via a two-stage sequential optimization scheme. Extensive experiments demonstrate that our method outperforms state-of-the-art techniques, validating flow matching as an effective alternative for domain adaptation.

%% file: sections/01_introduction.tex
\section{Introduction}

\begingroup
\setlength{\intextsep}{4pt}
\setlength{\columnsep}{12pt}
\begin{wrapfigure}[13]{r}{0.50\textwidth}
    \centering
    \includegraphics[width=\linewidth]{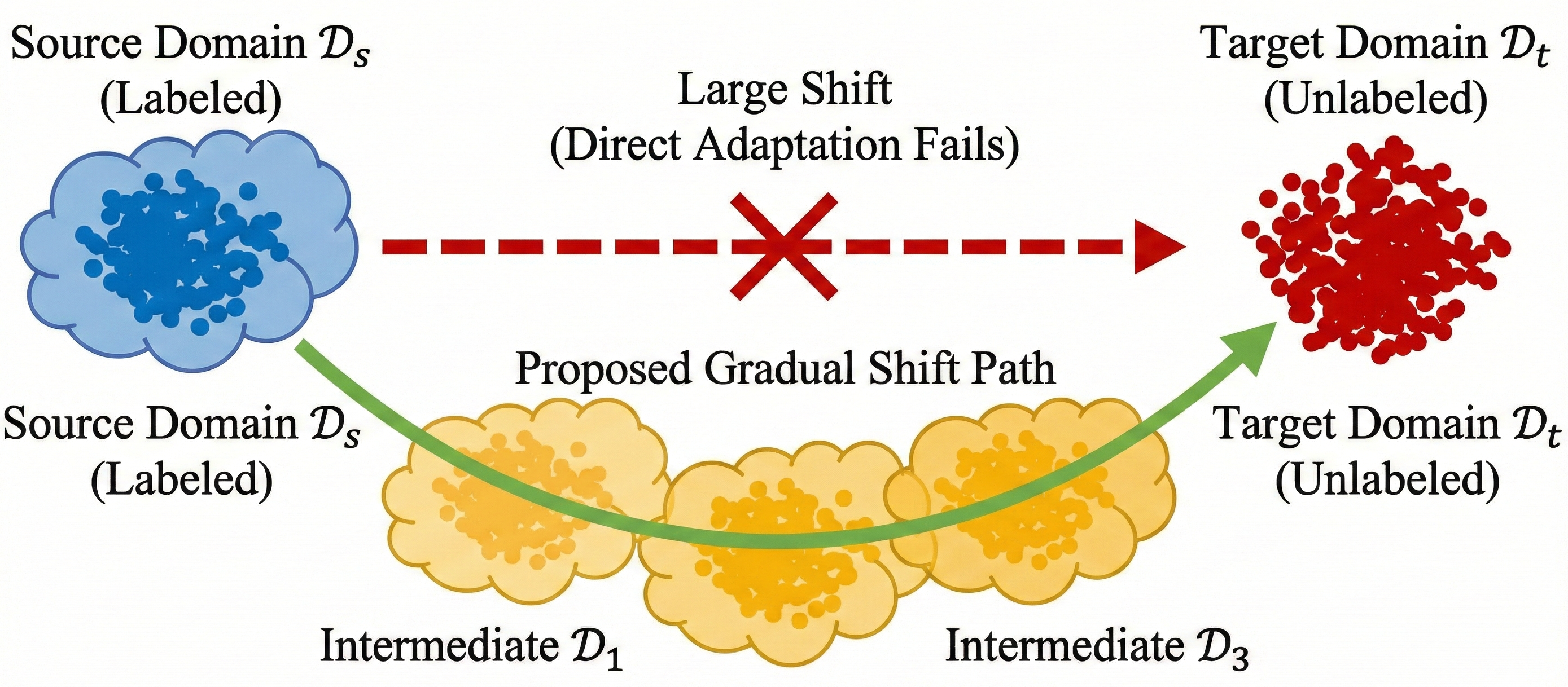}
    \captionsetup{font=small,skip=4pt}
    \caption{Motivation of gradual domain adaptation under large distribution shifts. Gradual domain adaptation introduces intermediate domains to form a smooth transition path, reducing the discrepancy between consecutive domains and facilitating stable knowledge transfer.}
    \label{fre1}
\end{wrapfigure}

Machine learning models often suffer from performance degradation under domain shift, where the test distribution differs from the training distribution~\cite{pan2010survey,bengio2013representation,bendavid2010theory}. Although traditional domain adaptation methods align a labeled source domain with an unlabeled target domain~\cite{ganin2016domain,long2015learning,tzeng2017adversarial}, direct alignment becomes unreliable when the distribution gap is large, as illustrated in Fig.~\ref{fre1}. Gradual domain adaptation (GDA) mitigates this issue by introducing ordered intermediate domains, decomposing a large shift into a sequence of smaller transitions. This setting has strong theoretical support, with intermediate domains shown to substantially improve generalization bounds over direct adaptation~\cite{wang2022understanding}.

\endgroup

Despite its promise, existing GDA methods typically treat domain evolution as a discrete sequence of adaptation steps. Such stepwise procedures may accumulate errors and fail to capture the smooth geometric structure of distribution shifts~\cite{li2025gradual}. Many methods also rely on adversarial training~\cite{ganin2016domain,tzeng2017adversarial} or complex optimal-transport optimization~\cite{villani2009optimal,peyré2019computational,xiao2026yang,yang2025automated}, which can be unstable or computationally demanding in high-dimensional representation spaces~\cite{goodfellow2014generative,flamary2024recent}. These limitations motivate a stable and efficient framework for learning continuous domain trajectories.

Flow matching~\cite{lipman2023flow} offers a natural alternative. By directly regressing velocity fields of continuous normalizing flows~\cite{chen2018neural,rezende2015variational,yao2025learning}, it avoids simulation-heavy or adversarial objectives, and its extensions to Riemannian manifolds~\cite{yao2022riemannian,chen2023riemannian} make it suitable for structured representation spaces. Motivated by this, we propose \emph{FlowGDA}, a flow-matching framework that views gradual domain adaptation as continuous transport on a statistical manifold~\cite{pooladian2023categorical}. Instead of adversarially aligning domains, FlowGDA learns velocity fields that move representations along smooth trajectories while preserving semantic information. Our contributions are summarized as follows:

\begin{itemize}
    \item We formulate gradual domain adaptation as flow matching on a statistical manifold, enabling continuous representation transport without adversarial training.
    \item We design a simulation-free objective with dual source-to-intermediate and intermediate-to-target flows, together with semantic consistency regularization to reduce drift.
    \item We provide theoretical analysis bounding the target error and show through experiments on multiple benchmarks that FlowGDA improves accuracy, stability, and efficiency over existing GDA methods.
\end{itemize}

%% file: sections/02_related_work.tex
\section{Related Work}



\subsection{Gradual Domain Adaptation}

Gradual domain adaptation addresses large domain gaps by exploiting ordered intermediate domains. Early theory showed that self-training over intermediate domains can outperform direct adaptation when consecutive shifts are small, while later analyses established tighter bounds and characterized optimal adaptation paths~\cite{wang2022understanding,zhang2019bridging}. Algorithmically, existing methods construct or exploit intermediate domains through optimal transport, manifold-constrained robust optimization~\cite{li2025gradual}, graph-based adaptation~\cite{wang2025gradual}, and domain osmosis~\cite{wang2026self,wang2025multi}. However, most approaches still treat the domain path as a discrete sequence, which may accumulate errors and underuse the smooth structure of gradual distribution evolution. Our work instead models this evolution as a continuous flow in representation space.

\subsection{Optimal Transport for Domain Adaptation}

Optimal transport provides a geometric framework for comparing and aligning probability distributions~\cite{villani2009optimal,peyré2019computational,yao2026disa}. It has been widely used in domain adaptation to learn transport plans that map source samples to target samples while preserving class structure~\cite{courty2015optimal}. Wasserstein-based methods further connect optimal transport with adversarial learning, including Wasserstein GANs~\cite{arjovsky2017wasserstein,gulrajani2017improved} and domain adaptation methods such as WDGRL~\cite{wdgrl2018}. Although these methods offer meaningful distributional geometry, they often require costly transport computation or adversarial optimization with difficult Lipschitz constraints, especially in high-dimensional feature spaces~\cite{flamary2024recent,zhao2025universal}. FlowGDA keeps the geometric transport perspective while replacing adversarial alignment with simulation-free velocity-field regression.

%% file: sections/03_method.tex
\section{Problem Formulation}

\noindent\textbf{Preliminary Note.}
In this manuscript, we focus on representation-level flow modeling for gradual domain adaptation (GDA). We follow the standard GDA assumption that the source domain, the intermediate domains, and the target domain are ordered along a gradual adaptation path. This assumption is used to decompose a large source-to-target shift into a sequence of smaller local shifts, rather than to introduce additional supervision on the intermediate or target domains.

Gradual domain adaptation considers a labeled source domain
$\mathcal{D}_s=\{(x_i^s,y_i^s)\}_{i=1}^{n_s}$, a sequence of unlabeled intermediate domains
$\{\mathcal{D}_1,\ldots,\mathcal{D}_{T-1}\}$, and an unlabeled target domain $\mathcal{D}_t$.
Here, $x_i^s\in\mathcal{X}$ is the $i$-th source input, $y_i^s\in\mathcal{Y}$ is its label,
$n_s$ is the number of labeled source samples, $\mathcal{X}$ denotes the input space, and
$\mathcal{Y}$ denotes the label space. The integer $T$ denotes the number of source-to-target
transition intervals, so that the complete domain chain can be written as
$\mathcal{D}_0,\mathcal{D}_1,\ldots,\mathcal{D}_T$, with
$\mathcal{D}_0=\mathcal{D}_s$ and $\mathcal{D}_T=\mathcal{D}_t$.
We use $\mathcal{D}_s$, $\mathcal{D}_i$, and $\mathcal{D}_t$ both to denote empirical domains
and, when they appear inside expectations, their underlying data distributions.

The domains are ordered such that consecutive distribution shifts are milder than the direct
source-to-target shift. The goal is to learn a classifier
$f=h\circ\phi:\mathcal{X}\rightarrow\mathcal{Y}$, where
$\phi:\mathcal{X}\rightarrow\mathcal{Z}$ maps an input into a representation space
$\mathcal{Z}$, $h:\mathcal{Z}\rightarrow\mathcal{Y}$ maps a representation to a label
prediction, and $\circ$ denotes function composition. The target error is defined as
\begin{equation}
	\epsilon_t=
	\mathbb{E}_{(x,y)\sim\mathcal{D}_t}
	\left[\mathbb{I}[f(x)\neq y]\right],
\end{equation}
where $(x,y)$ is an input-label pair from the target domain, $y$ is used only for evaluation,
$\mathbb{E}$ denotes expectation, and $\mathbb{I}[\cdot]$ is the indicator function that equals
one when its argument is true and zero otherwise. The subscript $t$ in $\epsilon_t$ denotes the
target domain, while the scalar $t\in[0,1]$ introduced later denotes continuous transport time.

\begin{wrapfigure}{r}{0.50\linewidth}
	\vspace{-14pt}
	\centering
	\includegraphics[width=\linewidth]{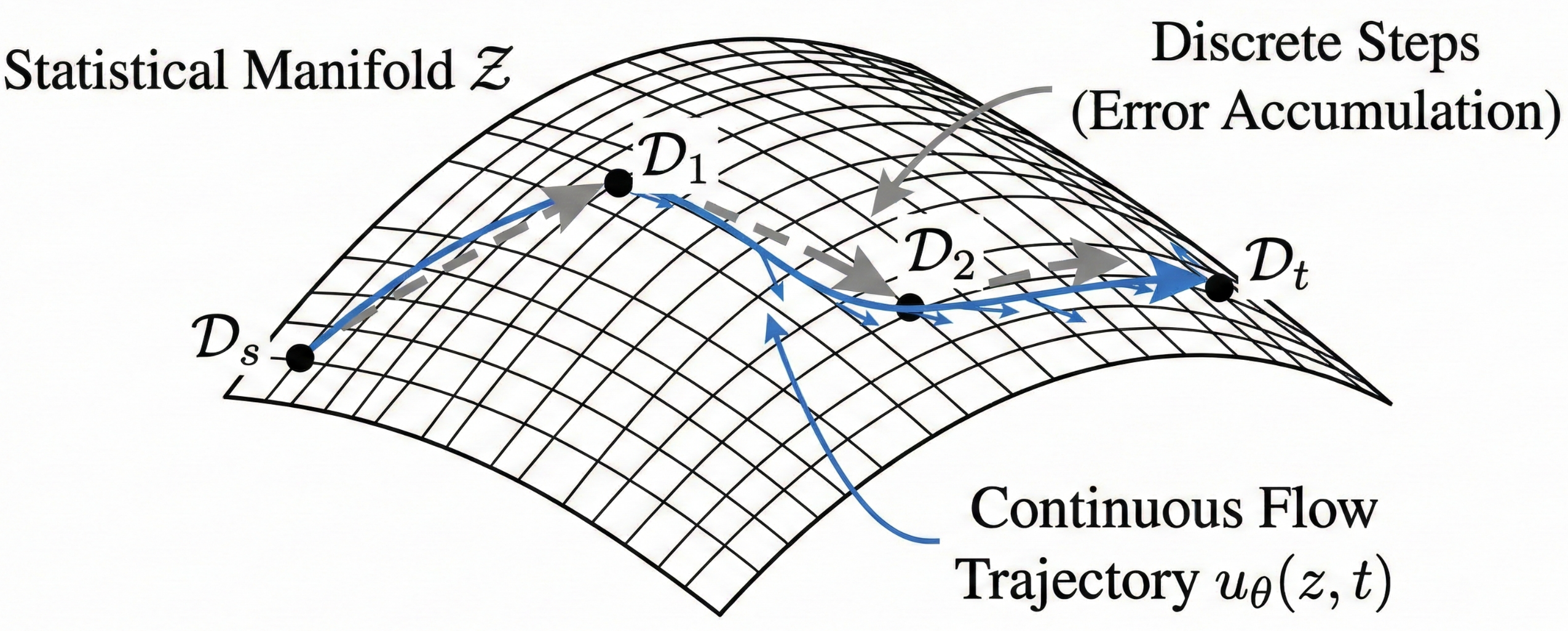}
	\vspace{-0.6em}
	\caption{Domain evolution on a statistical manifold. Discrete steps can accumulate errors, while the continuous flow $u_\theta(z,t)$ transports $\mathcal{D}_s$ toward $\mathcal{D}_t$.}
    \vspace{-30pt}
	\label{fre2}
\end{wrapfigure}

FlowGDA views representations from different domains as lying near a shared statistical manifold.
Under this view, domain shift corresponds to movement along the manifold, and gradual adaptation
amounts to learning smooth velocity fields that transport representations across neighboring
domains while preserving task-relevant semantics.

\section{Methodology}

FlowGDA formulates gradual domain adaptation as continuous transport in
representation space.As shown in Fig.~\ref{fre2}, discrete stepwise
adaptation may accumulate local misalignment along the domain path. FlowGDA
instead models domain evolution as a smooth trajectory on a statistical
manifold, governed by a time-dependent velocity field $u_\theta(z,t)$.

This velocity field is learned through flow matching, avoiding adversarial
optimization and trajectory simulation. Fig.~\ref{fig:framework} summarizes the
framework, including dual flow-matching losses, semantic consistency
regularization, and two-stage sequential optimization over consecutive domain
pairs.

\begin{figure*}[t]
	\centering
	\includegraphics[width=\textwidth]{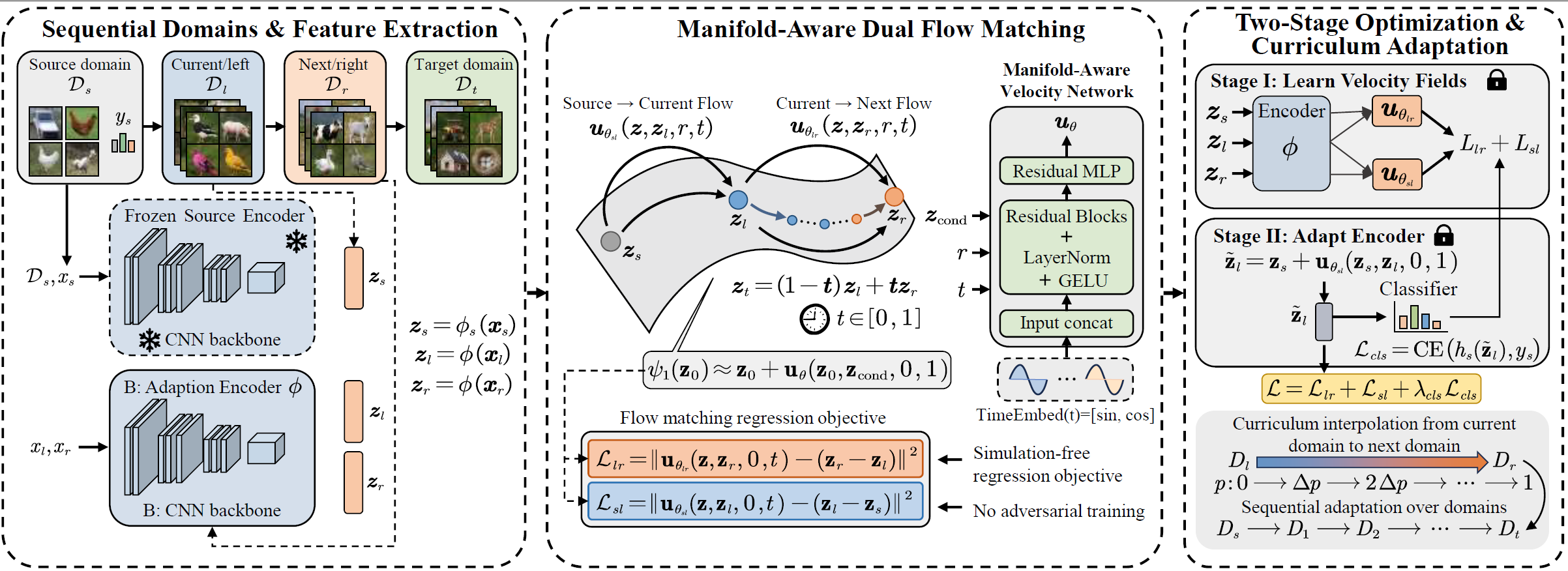}
	\caption{Overview of FlowGDA. FlowGDA models gradual domain transport with source-to-left and left-to-right velocity fields, enforces semantic consistency using a fixed source classifier, and trains sequentially over domain pairs with curriculum-based two-stage optimization.}
	\label{fig:framework}
\end{figure*}

\subsection{Flow Matching for Domain Transport}
\label{sec:method_fm}

We first describe the basic transport between two consecutive domains. Let $\mathcal{D}_l$ and
$\mathcal{D}_r$ denote the current ``left'' domain and the next ``right'' domain, respectively.
The subscripts $l$ and $r$ are mnemonic: $l$ refers to the departure domain in the current
adaptation step, and $r$ refers to its neighboring destination domain. For samples
$x_l\sim\mathcal{D}_l$ and $x_r\sim\mathcal{D}_r$, the trainable encoder
$\phi:\mathcal{X}\rightarrow\mathcal{Z}$ maps them to
$z_l=\phi(x_l),\ z_r=\phi(x_r),$
where $z_l,z_r\in\mathcal{Z}$ are their representations, $\mathcal{Z}\subseteq\mathbb{R}^d$
is the representation space, $\mathbb{R}^d$ is the $d$-dimensional Euclidean space, and $d$ is
the representation dimension.

A time-dependent velocity field
$u_\theta(z,t):\mathcal{Z}\times[0,1]\rightarrow\mathcal{Z}$ induces a continuous flow map
$\psi_t:\mathcal{Z}\rightarrow\mathcal{Z}$ through the ordinary differential equation
\begin{equation}
	\frac{d\psi_t(z_0)}{dt}
	=
	u_\theta(\psi_t(z_0),t),
	\qquad
	\psi_0(z_0)=z_0,
	\label{eq:method_ode}
\end{equation}
where $z_0\in\mathcal{Z}$ is the initial representation, $t\in[0,1]$ is the transport time,
$\psi_t(z_0)$ is the representation reached at time $t$, and $\theta$ denotes the trainable
parameters of the velocity network. The derivative $d\psi_t(z_0)/dt$ describes the instantaneous
movement of the representation along the flow.

Rather than simulating the ODE during training, FlowGDA adopts conditional flow matching. Given a
paired representation $(z_l,z_r)$, we define the deterministic interpolation path
\begin{equation}
	z_t=(1-t)z_l+t z_r,\qquad t\in[0,1].
	\label{eq:method_linear_path}
\end{equation}
Here, $z_t$ is the intermediate representation at transport time $t$; when $t=0$, the path starts
from $z_l$, and when $t=1$, the path reaches $z_r$. Equivalently, this interpolation induces the
degenerate conditional distribution
\begin{equation}
	p_t(z\mid z_l,z_r)
	=
	\delta\big(z-((1-t)z_l+t z_r)\big),
	\label{eq:method_dirac_path}
\end{equation}
where $p_t(z\mid z_l,z_r)$ denotes the conditional distribution of the intermediate representation
$z$ at time $t$, and $\delta(\cdot)$ is the Dirac delta distribution concentrated at the interpolated
point.

The velocity that exactly generates the path in Equation~\eqref{eq:method_linear_path} is
\begin{equation}
	u_t^\star(z_t\mid z_l,z_r)
	=
	\frac{d z_t}{dt}
	=
	z_r-z_l.
	\label{eq:method_cond_velocity}
\end{equation}
Here, $u_t^\star$ denotes the ideal conditional velocity, and $z_r-z_l$ is the displacement from the
left-domain representation to the right-domain representation. Based on this observation, the flow
matching loss is
\begin{equation}
	\mathcal{L}_{FM}
	=
	\mathbb{E}_{t\sim\operatorname{Unif}[0,1]}
	\mathbb{E}_{(z_l,z_r)\sim\pi_{lr}}
	\left[
	\left\|
	u_\theta(z_t,t)-(z_r-z_l)
	\right\|_2^2
	\right],
	\label{eq:method_fm_loss}
\end{equation}
where $\mathcal{L}_{FM}$ is the flow matching objective, $\operatorname{Unif}[0,1]$ is the uniform
distribution on $[0,1]$, $\pi_{lr}$ is the coupling between representations from $\mathcal{D}_l$
and $\mathcal{D}_r$, and $\|\cdot\|_2$ denotes the Euclidean norm.

\subsection{Manifold-Aware Flow Matching}
\label{sec:method_manifold}

Although Equation~\eqref{eq:method_linear_path} is written as a Euclidean interpolation, domain
representations usually concentrate near a lower-dimensional statistical manifold embedded in
$\mathcal{Z}$.A velocity field that only observes $z_t$ and $t$ may ignore where the current
transport step lies along the global source-to-target trajectory. To address this issue, FlowGDA
conditions the velocity field on both a destination representation and normalized domain coordinates.

For the ordered chain $\mathcal{D}_0,\mathcal{D}_1,\ldots,\mathcal{D}_T$, each domain is assigned
a normalized coordinate
$\tau_i=\frac{i}{T},
\qquad
i=0,\ldots,T,$
where $i$ is the domain index, $T$ is the total number of transition intervals, and
$\tau_i\in[0,1]$ indicates the relative position of $\mathcal{D}_i$ along the gradual path.
For a consecutive pair $\mathcal{D}_l=\mathcal{D}_i$ and
$\mathcal{D}_r=\mathcal{D}_{i+1}$, we write $\tau_l=\tau_i$ and
$\tau_r=\tau_{i+1}$. The source coordinate is denoted by $\tau_s=0$.

We parameterize the manifold-aware velocity field as
\begin{equation}
	u_\theta(z,z_{\mathrm{cond}},\tau_a,\tau_b,t):
	\mathcal{Z}\times\mathcal{Z}\times[0,1]\times[0,1]\times[0,1]
	\rightarrow
	\mathcal{Z},
	\label{eq:method_manifold_velocity}
\end{equation}
where $z$ is the current representation to be transported, $z_{\mathrm{cond}}\in\mathcal{Z}$ is
a conditioning representation sampled from the destination domain, $\tau_a$ is the normalized
coordinate of the departure domain, $\tau_b$ is the normalized coordinate of the destination domain,
$a$ and $b$ denote generic departure and destination indices, and $t$ is the local transport time.
The time and domain coordinates are embedded using sinusoidal features:
\begin{equation}
	\operatorname{TimeEmbed}(\tau)
	=
	\left[
	\sin(2\pi\omega_k\tau),
	\cos(2\pi\omega_k\tau)
	\right]_{k=1}^{K_e},
	\label{eq:method_time_embedding}
\end{equation}
where $\tau\in[0,1]$ is a scalar coordinate, $K_e$ is the number of frequency components,
$k$ indexes the frequency component, $\omega_k$ is the $k$-th frequency scale, and
$[\cdot]_{k=1}^{K_e}$ denotes concatenation over all frequency components. The symbol $\pi$ in
$2\pi$ is the circle constant and is distinct from the coupling $\pi_{lr}$ in
Equation~\eqref{eq:method_fm_loss}. In implementation, the embeddings of $\tau_a$, $\tau_b$, and
$t$ are used as conditioning signals for $u_\theta$.

\subsection{Dual Flow Architecture}
\label{sec:method_dual}

For each adaptation step, FlowGDA uses two velocity fields. The first velocity field,
denoted by $u_{\theta_{lr}}$, transports representations from the current domain
$\mathcal{D}_l$ to the next domain $\mathcal{D}_r$. The second velocity field, denoted by
$u_{\theta_{sl}}$, transports source representations from the source domain $\mathcal{D}_s$ to
the current domain $\mathcal{D}_l$. Here, $\theta_{lr}$ and $\theta_{sl}$ are the trainable
parameters of the left-to-right and source-to-left velocity networks, respectively.

The left-to-right flow captures the local distribution shift between neighboring domains, while the
source-to-left flow provides a semantic anchor from the labeled source domain. For a departure
representation $z_a$, a destination conditioning representation $z_b$, and the corresponding domain
coordinates $\tau_a$ and $\tau_b$, we approximate the endpoint of the transport by
\begin{equation}
	\psi_{a\to b}(z_a)
	\approx
	z_a+
	u_{\theta_{ab}}(z_a,z_b,\tau_a,\tau_b,1),
	\label{eq:method_single_step}
\end{equation}
where $\psi_{a\to b}$ is the transport map from domain coordinate $\tau_a$ to domain coordinate
$\tau_b$, $z_a\in\mathcal{Z}$ is the representation being transported, $z_b\in\mathcal{Z}$ is the
destination-side conditioning representation, $\theta_{ab}$ denotes the parameters of the velocity
field from domain $a$ to domain $b$, and the final argument $1$ denotes full transport time.

\subsection{Training Objective}
\label{sec:method_objective}

We now instantiate the above formulation for mini-batch training. At an adaptation step from
$\mathcal{D}_l$ to $\mathcal{D}_r$, let $(x_j^s,y_j^s)$ denote the $j$-th labeled source sample,
$x_j^l$ denote the $j$-th unlabeled sample from the current domain $\mathcal{D}_l$, and
$x_j^r$ denote the $j$-th unlabeled sample from the next domain $\mathcal{D}_r$. We compute
\begin{equation}
	z_j^s=\phi_s(x_j^s),
	\qquad
	z_j^l=\phi(x_j^l),
	\qquad
	z_j^r=\phi(x_j^r),
	\label{eq:method_batch_embeddings}
\end{equation}
where $\phi_s:\mathcal{X}\rightarrow\mathcal{Z}$ is the fixed source encoder,
$\phi:\mathcal{X}\rightarrow\mathcal{Z}$ is the trainable encoder, and
$z_j^s,z_j^l,z_j^r\in\mathcal{Z}$ are the source, current-domain, and next-domain representations. The label $y_j^s\in\mathcal{Y}$ is available only for the source sample.

For the left-to-right flow, the interpolated representation is
\begin{equation}
	z_{j,t}^{lr}
	=
	(1-t)z_j^l+t z_j^r,
	\qquad
	t\sim\operatorname{Unif}[0,1],
	\label{eq:method_zlr_t}
\end{equation}
where the superscript $lr$ denotes the left-to-right path, the subscript $j$ denotes the mini-batch
sample index, and $t$ denotes the sampled transport time. The corresponding flow matching objective is
\begin{equation}
	\mathcal{L}_{lr}
	=
	\mathbb{E}_{t,j}
	\left[
	\left\|
	u_{\theta_{lr}}(z_{j,t}^{lr},z_j^r,\tau_l,\tau_r,t)
	-
	(z_j^r-z_j^l)
	\right\|_2^2
	\right],
	\label{eq:method_llr}
\end{equation}
where $\mathcal{L}_{lr}$ is the left-to-right loss, $\mathbb{E}_{t,j}$ denotes expectation over
$t\sim\operatorname{Unif}[0,1]$ and the mini-batch index $j$, $z_j^r$ is the destination
conditioning representation, $\tau_l$ and $\tau_r$ are the normalized coordinates of the current
and next domains, and $z_j^r-z_j^l$ is the target displacement.

For the source-to-left flow, we similarly define
\begin{equation}
	z_{j,t}^{sl}
	=
	(1-t)z_j^s+t z_j^l,
	\qquad
	t\sim\operatorname{Unif}[0,1],
	\label{eq:method_zsl_t}
\end{equation}
where the superscript $sl$ denotes the source-to-left path from the fixed source representation
$z_j^s$ to the current-domain representation $z_j^l$. The source-to-left flow matching loss is
\begin{equation}
	\mathcal{L}_{sl}
	=
	\mathbb{E}_{t,j}
	\left[
	\left\|
	u_{\theta_{sl}}(z_{j,t}^{sl},z_j^l,\tau_s,\tau_l,t)
	-
	(z_j^l-z_j^s)
	\right\|_2^2
	\right],
	\label{eq:method_lsl}
\end{equation}
where $\mathcal{L}_{sl}$ is the source-to-left loss, $\tau_s=0$ is the source-domain coordinate,
$\tau_l$ is the current-domain coordinate, $z_j^l$ is the destination conditioning representation,
and $z_j^l-z_j^s$ is the target displacement from the source representation to the current-domain
representation.

To preserve semantic information during gradual adaptation, we introduce a curriculum coefficient
$p\in[0,1]$ that controls the strength of source-to-left transport:
\begin{equation}
	\tilde{z}_j^l(p)
	=
	z_j^s
	+
	p\cdot
	u_{\theta_{sl}}(z_j^s,z_j^l,\tau_s,\tau_l,1),
	\label{eq:method_partial_transport}
\end{equation}
where $\tilde{z}_j^l(p)\in\mathcal{Z}$ is the partially transported source representation. When
$p=0$, $\tilde{z}_j^l(p)=z_j^s$ and no source-to-left movement is applied; when $p=1$, the source
representation is fully transported toward the current domain according to
$u_{\theta_{sl}}$. The scalar $p$ in this curriculum term is distinct from the conditional path
distribution $p_t(\cdot\mid z_l,z_r)$ in Equation~\eqref{eq:method_dirac_path}.

The semantic consistency loss is defined as
\begin{equation}
	\mathcal{L}_{cls}(p)
	=
	\mathbb{E}_{j}
	\left[
	\operatorname{CE}
	\left(
	h_s(\tilde{z}_j^l(p)),
	y_j^s
	\right)
	\right],
	\label{eq:method_lcls}
\end{equation}
where $\mathcal{L}_{cls}(p)$ is the classification-based semantic loss at curriculum level $p$,
$\mathbb{E}_{j}$ denotes expectation over source samples in the mini-batch,
$h_s:\mathcal{Z}\rightarrow\mathcal{Y}$ is the fixed source classifier, and
$\operatorname{CE}(\cdot,\cdot)$ denotes the cross-entropy loss between the prediction
$h_s(\tilde{z}_j^l(p))$ and the source label $y_j^s$. The classifier $h_s$ serves as the semantic
head used to preserve source-discriminative structure during adaptation.

The total objective at curriculum level $p$ is
\begin{equation}
	\mathcal{L}(p)
	=
	\mathcal{L}_{lr}
	+
	\mathcal{L}_{sl}
	+
	\lambda_{cls}\mathcal{L}_{cls}(p),
	\label{eq:method_total_loss}
\end{equation}
where $\mathcal{L}(p)$ is the overall training loss and $\lambda_{cls}>0$ balances flow matching with semantic consistency.

\begingroup
\setlength{\columnsep}{16pt}
\setlength{\intextsep}{5pt}
\begin{wrapfigure}[12]{r}{0.48\linewidth}
	\vspace{-0.5em}
	\raggedleft
	\begin{minipage}{\linewidth}
		\small
		\linespread{0.96}\selectfont
		\hrule
		\vspace{1pt}
		\refstepcounter{algorithm}\label{alg:flowgda}
		\noindent\textbf{Algorithm~\thealgorithm} FlowGDA
		\addcontentsline{loa}{algorithm}{\protect\numberline{\thealgorithm}{FlowGDA}}
		\vspace{1pt}
		\hrule
		\vspace{1pt}
		\begin{algorithmic}[1]
			\REQUIRE $\mathcal{D}_{0:T}$, $\phi,h$, $u_{lr},u_{sl}$
			\STATE Cache: $\phi_s\leftarrow\phi$, $h_s\leftarrow h$
			\FOR{$i=0$ to $T-1$}
			\STATE $(\mathcal{D}_l,\mathcal{D}_r)\leftarrow(\mathcal{D}_i,\mathcal{D}_{i+1})$
			\STATE $u_{sl}\leftarrow u_{lr}$; reset $u_{lr}$
			\FOR{$p=0,\Delta p,\ldots,1$}
			\STATE Freeze $\phi$; update $u_{lr},u_{sl}$
			\STATE Freeze $u_{lr},u_{sl}$; update $\phi$
			\ENDFOR
			\ENDFOR
			\RETURN Adapted $\phi,h$
		\end{algorithmic}
		\vspace{1pt}
		\hrule
	\end{minipage}
	\vspace{-0.7em}
\end{wrapfigure}

Finally, FlowGDA proceeds sequentially along the ordered domains
$\mathcal{D}_0,\mathcal{D}_1,\ldots,\mathcal{D}_T$.For each consecutive pair
$(\mathcal{D}_l,\mathcal{D}_r)$, we first optimize the velocity fields
$u_{\theta_{lr}}$ and $u_{\theta_{sl}}$ while freezing the encoder $\phi$, so that the local
transport directions can be estimated in a stable representation space. We then freeze the learned
velocity fields and update the encoder $\phi$ using the curriculum-controlled objective in
Equation~\eqref{eq:method_total_loss}. During this stage, the curriculum coefficient $p$ is gradually
increased from $0$ to $1$, allowing the model to move from source-preserving representations toward
the next domain without abruptly losing semantic consistency. The complete procedure is summarized in
Algorithm~\ref{alg:flowgda}.
\par
\endgroup

\begin{table*}[!htbp]
	\centering
	\caption{Target-domain accuracy (\%) on four GDA benchmarks; ''N/A'' means not reported.}
	\label{tab:gda_main}
	\renewcommand{\arraystretch}{1.0}
	\resizebox{1\textwidth}{!}{%
		\begin{tabular}{llcccc}
			\toprule
			\textbf{Category} & \textbf{Method} & \textbf{Rotated MNIST} & \textbf{Color-Shift MNIST} & \textbf{Portraits} & \textbf{Cover Type} \\
			\midrule
			\multirow{3}{*}{UDA} 
			& DANN~\cite{ganin2016domain}       & 44.2 & 56.5 & 73.8 & N/A \\
			& DeepCORAL~\cite{sun2016deep}           & 49.6 & 63.5 & 71.9 & N/A \\
			& DeepJDOT~\cite{long2018conditional}  & 51.6 & 65.8 & 72.5 & N/A \\
			\midrule
			\multirow{4}{*}{Gradual} 
			& GST~\cite{kumar2020understanding}      & 83.8 & 74.0 & 82.6 & 73.5 \\
			& GOAT~\cite{he2023gradual}              & 86.4 & 91.8 & 83.6 & 69.9 \\
			& AST~\cite{shi2024adversarial}          & 90.6 & N/A  & 84.8 & N/A \\
			& CNF~\cite{sagawa2025gradual}           & 62.6 & N/A  & 84.6 & N/A \\
			\midrule
			\textbf{Ours} & \textbf{FlowGDA} & \textbf{97.0} & \textbf{99.2} & \textbf{87.6} & \textbf{76.1} \\
			\bottomrule
		\end{tabular}%
	}
\end{table*}

%% file: sections/04_experiments.tex
\section{Experiments}

\subsection{Experimental Setup}
\label{sec:exp_setup}

\paragraph{Datasets.} Following the standard gradual domain adaptation (GDA) protocol, FlowGDA is evaluated on four benchmarks covering complementary domain trajectories. Rotated MNIST captures smooth geometric shift through increasing rotation angles; Color-Shift MNIST captures appearance shift through gradual changes in color statistics; Portraits captures real-world temporal drift across decades; and Cover Type evaluates tabular covariate shift in structured feature space. For each benchmark, domains form an ordered chain $\mathcal{D}_0=\mathcal{D}_s,\mathcal{D}_1,\ldots,\mathcal{D}_T=\mathcal{D}_t$ according to the natural shift parameter. Only the source domain is labeled, while intermediate and target domains are unlabeled during adaptation. All methods are trained sequentially along the chain.

\paragraph{Baselines.} FlowGDA is compared with direct unsupervised domain adaptation baselines, including DANN~\cite{ganin2016domain}, DeepCORAL~\cite{sun2016deep}, and CDAN~\cite{long2018conditional}, as well as gradual adaptation methods including GST~\cite{kumar2020understanding}, GOAT~\cite{he2023gradual}, AST~\cite{shi2024adversarial}, and CNF~\cite{sagawa2025gradual}.

\paragraph{Architectures and optimization.} All methods are implemented in PyTorch. Image benchmarks use a three-layer CNN encoder followed by a two-layer fully connected classifier, while Cover Type uses a three-layer MLP. Velocity fields are implemented as 4-layer residual MLPs with hidden dimension 512. AdamW~\cite{loshchilov2019decoupled} is used with learning rates $10^{-4}$ for velocity networks and $10^{-5}$ for the encoder. The classification weight is set to $\lambda_{cls}=0.1$. Models are trained for 10 epochs per domain pair with batch size 64 on NVIDIA RTX 5090 GPUs.

\subsection{Results}

FlowGDA achieves the best performance across all four benchmarks. Compared with the strongest reported gradual baseline, FlowGDA improves target accuracy by +6.4 points on Rotated MNIST, +7.4 on Color-Shift MNIST, and +2.8 on Portraits; on Cover Type, it outperforms GST by +2.6 points. Direct UDA baselines perform substantially worse under large shifts, confirming the importance of intermediate domains. These results show that modeling gradual adaptation as continuous representation transport provides consistent gains over both direct alignment and discrete gradual adaptation.

\begin{figure*}[t]
\centering
\begin{tikzpicture}

\def\dataA{
1 0.942 0.018
2 0.958 0.010
3 0.961 0.008
4 0.966 0.007
5 0.963 0.009
6 0.972 0.008
7 0.959 0.021
}

\def\dataB{
1 0.962 0.018 
2 0.918 0.028 
}

\begin{groupplot}[
  group style={group size=2 by 1, horizontal sep=0.95cm},
  width=0.47\textwidth,
  height=0.28\textwidth,
  ymin=0.88, ymax=1.0,
  grid=both,
  grid style={line width=0.1pt, draw=black!10},
  major grid style={line width=0.2pt, draw=black!15},
  tick align=outside,
  tick style={black!60},
  axis line style={black!60},
  label style={black!80, font=\small},
  tick label style={font=\scriptsize},
  title style={black, font=\bfseries\small},
]

\nextgroupplot[
  title={(a) Intermediate-domain sensitivity},
  xlabel={Rotation angle of intermediate domain},
  ylabel={Accuracy},
  xtick={1,2,3,4,5,6,7},
  xticklabels={$\pi/21$,$\pi/10.5$,$\pi/7$,$\pi/6$,$\pi/5.25$,$\pi/4.2$,$\pi/3.5$},
  x tick label style={rotate=35, anchor=east, font=\scriptsize},
]

\addplot[name path=U, draw=none] coordinates {
(1,0.980) (2,0.988) (3,0.989) (4,0.993) (5,0.992) (6,0.995) (7,0.990)
};
\addplot[name path=L, draw=none] coordinates {
(1,0.944) (2,0.968) (3,0.973) (4,0.979) (5,0.974) (6,0.979) (7,0.948)
};
\addplot[fill=black!12, draw=none] fill between[of=U and L];

\addplot[
  thick,
  mark=*,
  mark size=1.8pt,
] coordinates {
(1,0.962) (2,0.978) (3,0.981) (4,0.986) (5,0.983) (6,0.987) (7,0.969)
};

\nextgroupplot[
  title={(b) Worst intermediate vs. no intermediate},
  xlabel={Setting},
  xtick={1,2},
  xticklabels={Worst $\mathcal{D}_1$, w/o $\mathcal{D}_1$},
  x tick label style={font=\scriptsize},
  ymajorgrids=true,
  bar width=22pt,
   xmin=0.5, xmax=2.5,
   enlarge x limits=0.3,
]

\addplot+[
  ybar,
  bar width=20pt,
  error bars/.cd,
    y dir=both,
    y explicit,
] coordinates {
(1,0.942) +- (0,0.018)
(2,0.918) +- (0,0.028)
};

\end{groupplot}
\end{tikzpicture}
\caption{Rotated MNIST sensitivity to intermediate-domain choice. (a) Accuracy vs. intermediate rotation (mean$\pm$std over five runs). (b) Removing the intermediate domain causes the largest degradation.}
\vspace{-10pt}
\label{fig:intermediate_sensitivity}
\end{figure*}
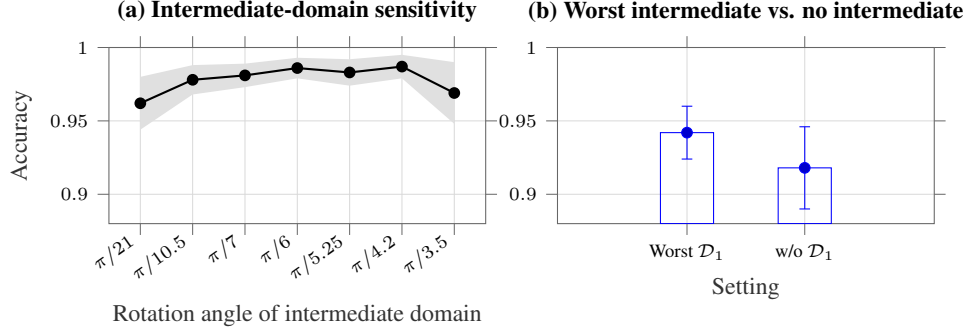

On Rotated MNIST, the effect of intermediate-domain choice is further evaluated by fixing the source at $0^\circ$ and the target at $\pi/3$, while varying a single intermediate domain $\mathcal{D}_1$. As shown in Fig.~\ref{fig:intermediate_sensitivity}(a), intermediate angles near the midpoint of the source-target path, around $\pi/6$, yield the highest and most stable accuracy, whereas very small or overly large intermediate shifts lead to lower accuracy and higher variance. Fig.~\ref{fig:intermediate_sensitivity}(b) further shows that direct source-to-target adaptation performs worst, indicating that intermediate domains are essential for large shifts, while their quality further affects stability and final performance.

\subsection{Ablation Study}
\label{sec:ablation}

To better understand how different components contribute across different types of domain shifts, we conduct a dataset-wise ablation study by removing one component at a time from FlowGDA. Table~\ref{tab:ablation} reports the target accuracy on each benchmark, with the accuracy drop relative to the full model shown in parentheses. The effect of each component is not uniform across datasets. Removing the semantic consistency loss $\mathcal{L}_{cls}$ causes the largest degradation on average, especially on Portraits and Cover Type, where preserving class semantics under complex temporal or tabular covariate shifts is more challenging. The single-flow variant also suffers larger drops on the real-world and tabular datasets, indicating that source-anchored transport provides useful regularization beyond local left-to-right adaptation. Manifold-aware conditioning improves over time-only velocity parameterization across all datasets, while the two-stage optimization strategy consistently yields more stable final performance than joint optimization.

\begin{table*}[t]
\centering
\caption{Dataset-wise ablation study: target accuracy (\%) with drops from full FlowGDA in parentheses.}
\label{tab:ablation}
\renewcommand{\arraystretch}{1.0}
\resizebox{0.95\textwidth}{!}{%
\begin{tabular}{lccccc}
\toprule
\textbf{Variant} 
& \textbf{Rotated MNIST} 
& \textbf{Color-Shift MNIST} 
& \textbf{Portraits} 
& \textbf{Cover Type} 
& \textbf{Avg. $\Delta$ (pp)} \\
\midrule
FlowGDA (full) 
& \textbf{97.0} 
& \textbf{99.2} 
& \textbf{87.6} 
& \textbf{76.1} 
& -- \\
\midrule
Single flow 
& 95.5 {\small($-1.5$)} 
& 98.0 {\small($-1.2$)} 
& 85.6 {\small($-2.0$)} 
& 73.6 {\small($-2.5$)} 
& $-1.8$ \\
w/o semantic loss $\mathcal{L}_{cls}$ 
& 95.0 {\small($-2.0$)} 
& 97.6 {\small($-1.6$)} 
& 84.8 {\small($-2.8$)} 
& 73.3 {\small($-2.8$)} 
& $-2.3$ \\
Time-only velocity 
& 95.8 {\small($-1.2$)} 
& 98.2 {\small($-1.0$)} 
& 85.7 {\small($-1.9$)} 
& 74.2 {\small($-1.9$)} 
& $-1.5$ \\
Joint optimization 
& 96.3 {\small($-0.7$)} 
& 98.6 {\small($-0.6$)} 
& 86.5 {\small($-1.1$)} 
& 74.9 {\small($-1.2$)} 
& $-0.9$ \\
\bottomrule
\end{tabular}%
}
\vspace{-8pt}
\end{table*}

\section{Theoretical Analysis}
\label{sec:theory}

In this section, we theoretically explain why modeling gradual domain adaptation as a \emph{continuous transport trajectory} leads to stable optimization and reduced error accumulation. Our analysis connects flow matching to optimal transport on representation distributions and derives a generalization bound for the target domain under gradual shifts.

\subsection{Generalization Error Bounds for Flow-Based Gradual Adaptation}
\label{sec:theory_bound}

Let $\mathcal{D}_0=\mathcal{D}_s,\mathcal{D}_1,\ldots,\mathcal{D}_{T}=\mathcal{D}_t$ be ordered domains.
For each domain $\mathcal{D}_i$, let $P_i^X$ denote the marginal distribution over inputs and
$P_i^Z = (\phi)_\# P_i^X$ the pushforward distribution in representation space induced by encoder $\phi$.

We analyze expected risk with a Lipschitz surrogate loss.
Let $\ell(h(z),y)$ be a bounded loss (e.g., cross-entropy or margin loss), and assume it is $L_\ell$-Lipschitz w.r.t.\ $z$ for hypotheses in $\mathcal{H}$.

\begin{definition}[Wasserstein Discrepancy in Representation Space]
\label{def:w1}
For two distributions $P,Q$ on $\mathcal{Z}$, the Wasserstein-1 distance is
\begin{equation}
W_1(P,Q) = \sup_{\|f\|_{\mathrm{Lip}}\le 1} \left|\mathbb{E}_{z\sim P}[f(z)] - \mathbb{E}_{z\sim Q}[f(z)]\right|.
\end{equation}
\end{definition}

\begin{lemma}[Risk Difference Controlled by $W_1$]
\label{lem:risk_w1}
Assume $\ell(h(\cdot),y)$ is $L_\ell$-Lipschitz in $z$ for all $y\in\mathcal{Y}$ and $h\in\mathcal{H}$.
Then for any $h\in\mathcal{H}$ and any two domains $i,j$,
\begin{equation}
\left|\epsilon_i(h)-\epsilon_j(h)\right|
\;\le\;
L_\ell\, W_1(P_i^Z,P_j^Z),
\end{equation}
where $\epsilon_i(h)=\mathbb{E}_{(x,y)\sim \mathcal{D}_i}\left[\ell(h(\phi(x)),y)\right]$.
\end{lemma}

Lemma~\ref{lem:risk_w1} implies that controlling the Wasserstein discrepancy between representation distributions directly bounds the change in expected risk across domains.

\begin{theorem}[Target Risk Bound under Gradual Shifts]
\label{thm:gda_w1}
Let $\mathcal{D}_0,\ldots,\mathcal{D}_T$ be ordered domains and $h\in\mathcal{H}$.
Then the target risk satisfies
\begin{equation}
\epsilon_T(h)
\;\le\;
\epsilon_0(h)
+
L_\ell \sum_{i=0}^{T-1} W_1(P_i^Z,P_{i+1}^Z)
+
\lambda^\star,
\end{equation}
where $\lambda^\star=\min_{h\in\mathcal{H}}\left(\epsilon_0(h)+\epsilon_T(h)\right)$ is the ideal joint error term.
\end{theorem}

\begin{remark}[Error Accumulation in Discrete Adaptation]
\label{rem:accum}
Theorem~\ref{thm:gda_w1} makes explicit that gradual adaptation succeeds when the \emph{sum of consecutive discrepancies} is small.
Methods that perform discrete, independently-estimated alignments can introduce local mismatch at each step, effectively enlarging $\sum_i W_1(P_i^Z,P_{i+1}^Z)$ and causing accumulation.
Our continuous-flow modeling aims to reduce such stepwise mismatch by learning a coherent transport trajectory.
\end{remark}

\subsection{Flow Matching Controls Wasserstein Transport Cost}
\label{sec:theory_ot}

We now connect our flow objective to optimal transport. Let $(P,Q)$ be two distributions on $\mathcal{Z}$ and consider a time-dependent velocity field $u(z,t)$ generating a flow $\psi_t$ by
\begin{equation}
\frac{d\psi_t(z)}{dt} = u(\psi_t(z),t),\quad \psi_0(z)=z.
\end{equation}

\begin{definition}[Benamou--Brenier Transport Energy]
\label{def:bb}
For a velocity field $u$ that transports $P$ to $Q$ along a path $\{P_t\}_{t\in[0,1]}$, define the transport energy
\begin{equation}
\mathcal{E}(u) = \int_0^1 \mathbb{E}_{z\sim P_t}\big[\|u(z,t)\|_2^2\big]\,dt.
\end{equation}
\end{definition}

\begin{lemma}[Dynamic OT Upper Bound]
\label{lem:bb_upper}
For any admissible transport $(P_t,u)$ from $P$ to $Q$, we have
\begin{equation}
W_2^2(P,Q) \;\le\; \mathcal{E}(u).
\end{equation}
\end{lemma}

Lemma~\ref{lem:bb_upper} indicates that learning a low-energy velocity field implies a small Wasserstein-2 distance.

\begin{theorem}[Per-Step Discrepancy Control via Flow Matching]
\label{thm:fm_controls_w2}
Consider two consecutive domains $\mathcal{D}_i$ and $\mathcal{D}_{i+1}$ with representation distributions $P_i^Z$ and $P_{i+1}^Z$.
Assume the learned velocity field $u_{\theta,i}$ induces a transport path $(P_t,u_{\theta,i})$ between them.
Then
\begin{equation}
W_2^2(P_i^Z,P_{i+1}^Z) \;\le\; \mathcal{E}(u_{\theta,i}).
\end{equation}
Moreover, minimizing the flow matching objective yields a small empirical estimate of $\mathcal{E}(u_{\theta,i})$ up to approximation and estimation errors.
\end{theorem}

\begin{corollary}[From Flow Loss to Target Risk]
\label{cor:flow_to_risk}
Under the assumptions of Theorem~\ref{thm:gda_w1} and Theorem~\ref{thm:fm_controls_w2}, if $\mathcal{E}(u_{\theta,i})\le \eta_i$ for all $i$, then
\begin{equation}
\begin{aligned}
\epsilon_T(h)
\;\le\;
\epsilon_0(h)
+
L_\ell \sum_{i=0}^{T-1} W_1(P_i^Z,P_{i+1}^Z)
+
\lambda^\star,
\text{with } \\
W_1(P_i^Z,P_{i+1}^Z)\;\le\; W_2(P_i^Z,P_{i+1}^Z)\;\le\;\sqrt{\eta_i}.    
\end{aligned}
\end{equation}
\end{corollary}

\subsection{Semantic Consistency via Source-Anchored Transport}
\label{sec:theory_semantic}

Our training includes the source-anchored classification term $\mathcal{L}_{cls}$ to preserve label information during transport.
Let $p_y(z)$ denote the probability assigned to the true class $y$ by the fixed source classifier $h_s$.

\begin{lemma}[Classification Loss Bounds Misclassification Probability]
\label{lem:ce_bound}
Let $\ell_{\mathrm{CE}}(p_y)=-\log p_y$ be the per-sample cross-entropy loss.
Then for any random pair $(z,y)$,
\begin{equation}
\mathbb{P}\big[\arg\max_k h_s(z)_k \neq y\big]
\;\le\;
\mathbb{E}\big[1-p_y(z)\big]
\;\le\;
\mathbb{E}\big[\ell_{\mathrm{CE}}(p_y(z))\big].
\end{equation}
\end{lemma}

\begin{theorem}[Label Preservation under Source-to-Intermediate Flow]
\label{thm:label_preserve}
Let $\tilde{z}=\psi_{s\to i}(z_s)$ be the transported representation from the source to an intermediate domain induced by $u_{\theta_{sl}}$,
and let $\mathcal{L}_{cls}=\mathbb{E}_{(z_s,y)\sim P_s^Z}\left[\ell_{\mathrm{CE}}(h_s(\tilde{z}),y)\right]$.
Then the misclassification probability after transport is bounded by
\begin{equation}
\mathbb{P}\big[\arg\max_k h_s(\tilde{z})_k \neq y\big]\;\le\;\mathcal{L}_{cls}.
\end{equation}
\end{theorem}

Theorem~\ref{thm:label_preserve} formalizes that minimizing $\mathcal{L}_{cls}$ preserves semantic structure along the learned transport.

\subsection{Why Flow Matching is More Stable than Adversarial Alignment}
\label{sec:theory_stability}

Adversarial domain alignment solves a min--max problem, whose convergence depends on discriminator optimality and Lipschitz enforcement, often leading to oscillatory dynamics in high-dimensional representation spaces.
In contrast, flow matching reduces transport learning to a single-level regression objective, producing unbiased gradients and avoiding saddle-point instability.

\begin{remark}[Optimization Landscape]
\label{rem:opt}
Flow matching optimizes $\min_\theta \mathbb{E}\|u_\theta(z,t)-u^\star(z,t)\|^2$, a standard supervised regression objective, whereas adversarial alignment optimizes $\min_\phi \max_D \mathcal{J}(\phi,D)$, which can exhibit multiple equilibria and sensitivity to critic capacity.
This difference explains the empirically observed stability and faster convergence of our method.
\end{remark}

%% file: sections/05_conclusion.tex
\section{Conclusions}

This work introduced FlowGDA, a flow-matching framework for gradual domain adaptation. By modeling domain evolution as continuous transport on a statistical manifold, FlowGDA learns velocity fields that move representations across domains through a simulation-free regression objective, avoiding adversarial optimization. Experiments on multiple benchmarks show that FlowGDA consistently outperforms existing gradual adaptation methods, demonstrating the effectiveness of continuous representation transport for large domain shifts.

%% file: sections/acknowledgments.tex
This research was supported by the National Key R\&D Program of China (No.~2024YFB2606101), the National Natural Science Foundation of China (No.~52478418), and the Sichuan Science and Technology Program (No.~2025ZNSFSC1294).

%% file: sections/competing_interests.tex
The authors declare no competing interests.

%% file: sections/06_limitations.tex
\section*{Limitations}

FlowGDA relies on ordered intermediate domains that form a meaningful gradual path between the source and target distributions. When such domains are missing, unordered, or contain abrupt semantic changes, the learned transport may no longer approximate a smooth domain trajectory. The method also assumes that class semantics are largely preserved along the path, so severe label shift or class-appearance changes can weaken the source-anchored consistency signal. In addition, although training is simulation-free with respect to ODE integration, FlowGDA introduces extra velocity networks and sequential optimization over domain pairs, which may increase implementation complexity and training cost compared with simpler direct adaptation baselines.

%% file: sections/appendix/appendix.tex
\section{Theoretical Proofs}
\label{sec:proofs}

In this appendix, we provide detailed proofs for the theoretical claims in Section~\ref{sec:theory}. We first define all notation used in the derivations. We then prove the Wasserstein-based risk bound for gradual domain adaptation, derive the continuity equation for the interpolation path used in flow matching, show that flow matching learns the corresponding marginal velocity field, and finally prove the semantic consistency bound induced by the source-anchored classification loss.

\subsection{Notation and Standing Assumptions}
\label{sec:proofs_notation}

Let $\mathcal{X}$ denote the input space, $\mathcal{Y}=\{1,\ldots,K\}$ the label space with $K$ classes, and $\mathcal{Z}\subseteq\mathbb{R}^d$ the representation space of dimension $d$. A domain $\mathcal{D}_i$ is a probability distribution over $\mathcal{X}\times\mathcal{Y}$. We write $P_i^X$ for the marginal distribution of inputs in domain $\mathcal{D}_i$.

The encoder is denoted by
$\phi:\mathcal{X}\rightarrow\mathcal{Z},$
where $\phi$ maps an input $x\in\mathcal{X}$ to a representation $z=\phi(x)\in\mathcal{Z}$. The pushforward representation distribution induced by $\phi$ is
$P_i^Z=(\phi)_\#P_i^X,$
where $(\phi)_\#P_i^X$ means that if $x\sim P_i^X$, then $z=\phi(x)\sim P_i^Z$.

Let $h:\mathcal{Z}\rightarrow\mathcal{Y}$ be a classifier and let $\ell(h(z),y)$ be a surrogate loss, such as cross-entropy or a margin loss. The expected risk of $h$ on domain $\mathcal{D}_i$ is
\begin{equation}
\epsilon_i(h)
=
\mathbb{E}_{(x,y)\sim\mathcal{D}_i}
\left[
\ell(h(\phi(x)),y)
\right],
\label{eq:proof_risk_def}
\end{equation}
where the expectation is taken over labeled samples $(x,y)$ drawn from $\mathcal{D}_i$.

We consider an ordered gradual domain chain
$\mathcal{D}_0=\mathcal{D}_s,\
\mathcal{D}_1,\ldots,\mathcal{D}_{T-1},\
\mathcal{D}_T=\mathcal{D}_t,$
where $\mathcal{D}_s$ is the labeled source domain, $\mathcal{D}_t$ is the unlabeled target domain, and $\mathcal{D}_1,\ldots,\mathcal{D}_{T-1}$ are unlabeled intermediate domains. The integer $T$ denotes the number of transitions from source to target.

\paragraph{Shared-labeling assumption.}
We assume that the conditional label distribution in representation space is invariant across domains:
$P_i(y\mid z)=P_j(y\mid z),
\
\forall i,j\in\{0,\ldots,T\},
\
\forall z\in\mathcal{Z}.$
Here $P_i(y\mid z)$ is the conditional probability of label $y$ given representation $z$ in domain $\mathcal{D}_i$. This assumption means that the marginal distribution of representations may shift across domains, but the semantic labeling rule given $z$ is shared.

For any classifier $h$, define the conditional risk function
\begin{equation}
g_h(z)
=
\mathbb{E}_{y\sim P(y\mid z)}
\left[
\ell(h(z),y)
\right],
\label{eq:proof_gh_def}
\end{equation}
where $P(y\mid z)$ denotes the shared conditional label distribution. We assume that $g_h$ is $L_\ell$-Lipschitz:
\begin{equation}
|g_h(z)-g_h(z')|
\le
L_\ell\|z-z'\|_2,
\qquad
\forall z,z'\in\mathcal{Z},
\label{eq:proof_gh_lip}
\end{equation}
where $L_\ell>0$ is the Lipschitz constant and $\|\cdot\|_2$ is the Euclidean norm.

\paragraph{Wasserstein distances.}
For probability distributions $P$ and $Q$ on $\mathcal{Z}$, the Wasserstein-1 distance is
\begin{equation}
W_1(P,Q)
=
\sup_{\|f\|_{\mathrm{Lip}}\le 1}
\left|
\mathbb{E}_{z\sim P}[f(z)]
-
\mathbb{E}_{z\sim Q}[f(z)]
\right|,
\label{eq:proof_w1_def}
\end{equation}
where the supremum is over all 1-Lipschitz functions $f:\mathcal{Z}\rightarrow\mathbb{R}$.

The squared Wasserstein-2 distance is
\begin{equation}
W_2^2(P,Q)
=
\inf_{\pi\in\Gamma(P,Q)}
\mathbb{E}_{(z,z')\sim\pi}
\left[
\|z-z'\|_2^2
\right],
\label{eq:proof_w2_def}
\end{equation}
where $\Gamma(P,Q)$ denotes the set of all couplings whose first marginal is $P$ and whose second marginal is $Q$. A coupling $\pi\in\Gamma(P,Q)$ is a joint distribution over pairs $(z,z')$ such that $z\sim P$ and $z'\sim Q$.

\subsection{Proof of Lemma~\ref{lem:risk_w1}}
\label{sec:proof_risk_w1}

\begin{proof}
Fix any classifier $h\in\mathcal{H}$. By definition of the expected risk,
\begin{equation}
\epsilon_i(h)
=
\mathbb{E}_{(x,y)\sim\mathcal{D}_i}
\left[
\ell(h(\phi(x)),y)
\right].
\label{eq:proof_risk_step1}
\end{equation}
Let $z=\phi(x)$. Since $P_i^Z=(\phi)_\#P_i^X$, the representation $z$ follows $P_i^Z$ whenever $x\sim P_i^X$. Therefore,
\begin{equation}
\epsilon_i(h)
=
\mathbb{E}_{(z,y)\sim P_i^{Z,Y}}
\left[
\ell(h(z),y)
\right],
\label{eq:proof_risk_step2}
\end{equation}
where $P_i^{Z,Y}$ is the joint distribution of representation-label pairs induced by $\phi$ on domain $\mathcal{D}_i$.

Using the law of total expectation with respect to $z$, we obtain
\begin{equation}
\begin{aligned}
\epsilon_i(h)
&=
\int_{\mathcal{Z}}
\left(
\int_{\mathcal{Y}}
\ell(h(z),y)\,dP_i(y\mid z)
\right)
dP_i^Z(z).
\end{aligned}
\label{eq:proof_risk_step3}
\end{equation}
In~\eqref{eq:proof_risk_step3}, $dP_i^Z(z)$ denotes integration with respect to the representation distribution of the domain $\mathcal{D}_i$, and $dP_i(y\mid z)$ denotes integration over the conditional label distribution.

By the assumption of shared-labeling, $P_i(y\mid z)=P(y\mid z)$ does not depend on the domain index $i$. Hence
\begin{equation}
\begin{aligned}
\epsilon_i(h)
&=
\int_{\mathcal{Z}}
\left(
\int_{\mathcal{Y}}
\ell(h(z),y)\,dP(y\mid z)
\right)
dP_i^Z(z)\\
&=
\int_{\mathcal{Z}}
g_h(z)\,dP_i^Z(z),
\end{aligned}
\label{eq:proof_risk_step4}
\end{equation}
where $g_h$ is the conditional risk function defined in~\eqref{eq:proof_gh_def}. Now consider two domains $\mathcal{D}_i$ and $\mathcal{D}_j$. Using~\eqref{eq:proof_risk_step4},
\begin{equation}
\begin{aligned}
\epsilon_i(h)-\epsilon_j(h)
&=
\int_{\mathcal{Z}}g_h(z)\,dP_i^Z(z)
-
\int_{\mathcal{Z}}g_h(z)\,dP_j^Z(z)\\
&=
\mathbb{E}_{z\sim P_i^Z}[g_h(z)]
-
\mathbb{E}_{z\sim P_j^Z}[g_h(z)].
\end{aligned}
\label{eq:proof_risk_diff}
\end{equation}
Taking the absolute value,
\begin{equation}
\left|
\epsilon_i(h)-\epsilon_j(h)
\right|
=
\left|
\mathbb{E}_{z\sim P_i^Z}[g_h(z)]
-
\mathbb{E}_{z\sim P_j^Z}[g_h(z)]
\right|.
\label{eq:proof_risk_abs}
\end{equation}

Since $g_h$ is $L_\ell$-Lipschitz by~\eqref{eq:proof_gh_lip}, define the normalized function
\begin{equation}
\bar{g}_h(z)
=
\frac{g_h(z)}{L_\ell}.
\label{eq:proof_normalized_gh}
\end{equation}
Then $\bar{g}_h$ is 1-Lipschitz, i.e.,
$\|\bar{g}_h\|_{\mathrm{Lip}}\le 1.$
Substituting $g_h=L_\ell\bar{g}_h$ into~\eqref{eq:proof_risk_abs}, we get
\begin{equation}
\begin{aligned}
\left|
\epsilon_i(h)-\epsilon_j(h)
\right|
&=
L_\ell
\left|
\mathbb{E}_{z\sim P_i^Z}[\bar{g}_h(z)]
-
\mathbb{E}_{z\sim P_j^Z}[\bar{g}_h(z)]
\right|\\
&\le
L_\ell
\sup_{\|f\|_{\mathrm{Lip}}\le 1}
\left|
\mathbb{E}_{z\sim P_i^Z}[f(z)]
-
\mathbb{E}_{z\sim P_j^Z}[f(z)]
\right|\\
&=
L_\ell W_1(P_i^Z,P_j^Z).
\end{aligned}
\label{eq:proof_risk_w1_final}
\end{equation}
The last equality follows from the definition of $W_1$ in~\eqref{eq:proof_w1_def}. This proves Lemma~\ref{lem:risk_w1}.
\end{proof}

\subsection{Proof of Theorem~\ref{thm:gda_w1}}
\label{sec:proof_gda_bound}

\begin{proof}
Fix any classifier $h\in\mathcal{H}$. For each consecutive pair of domains $(\mathcal{D}_i,\mathcal{D}_{i+1})$, decompose the risk as
\begin{equation}
\epsilon_{i+1}(h)
=
\epsilon_i(h)
+
\left(
\epsilon_{i+1}(h)-\epsilon_i(h)
\right).
\label{eq:proof_gda_step1}
\end{equation}
Using the inequality $a\le |a|$ for any scalar $a$ when deriving an upper bound,
\begin{equation}
\epsilon_{i+1}(h)
\le
\epsilon_i(h)
+
\left|
\epsilon_{i+1}(h)-\epsilon_i(h)
\right|.
\label{eq:proof_gda_step2}
\end{equation}
Applying Lemma~\ref{lem:risk_w1} to $P_i^Z$ and $P_{i+1}^Z$ gives
\begin{equation}
\left|
\epsilon_{i+1}(h)-\epsilon_i(h)
\right|
\le
L_\ell W_1(P_i^Z,P_{i+1}^Z).
\label{eq:proof_gda_step3}
\end{equation}
Substituting~\eqref{eq:proof_gda_step3} into~\eqref{eq:proof_gda_step2}, we have
\begin{equation}
\epsilon_{i+1}(h)
\le
\epsilon_i(h)
+
L_\ell W_1(P_i^Z,P_{i+1}^Z).
\label{eq:proof_gda_step4}
\end{equation}

Applying~\eqref{eq:proof_gda_step4} recursively along the domain chain yields
\begin{equation}
\begin{aligned}
\epsilon_T(h)
&\le
\epsilon_{T-1}(h)
+
L_\ell W_1(P_{T-1}^Z,P_T^Z)\\
&\le
\epsilon_{T-2}(h)
+
L_\ell W_1(P_{T-2}^Z,P_{T-1}^Z)
+
L_\ell W_1(P_{T-1}^Z,P_T^Z)\\
&\le
\epsilon_{T-3}(h)
+
L_\ell W_1(P_{T-3}^Z,P_{T-2}^Z)
+
L_\ell W_1(P_{T-2}^Z,P_{T-1}^Z)
+
L_\ell W_1(P_{T-1}^Z,P_T^Z)\\
&\ \ \vdots\\
&\le
\epsilon_0(h)
+
L_\ell
\sum_{i=0}^{T-1}
W_1(P_i^Z,P_{i+1}^Z).
\end{aligned}
\label{eq:proof_gda_telescoping}
\end{equation}
Here $\epsilon_0(h)$ is the source-domain risk and $\epsilon_T(h)$ is the target-domain risk.

To account for irreducible mismatch between the source and target labeling functions, define the ideal joint error
\begin{equation}
\lambda^\star
=
\min_{h\in\mathcal{H}}
\left[
\epsilon_0(h)+\epsilon_T(h)
\right],
\label{eq:proof_lambda_def}
\end{equation}
where $\lambda^\star\ge0$ measures the minimum simultaneous source-target risk achievable by the hypothesis class $\mathcal{H}$. Adding this nonnegative term to the right-hand side of~\eqref{eq:proof_gda_telescoping} gives
\begin{equation}
\epsilon_T(h)
\le
\epsilon_0(h)
+
L_\ell
\sum_{i=0}^{T-1}
W_1(P_i^Z,P_{i+1}^Z)
+
\lambda^\star.
\label{eq:proof_gda_final}
\end{equation}
This proves Theorem~\ref{thm:gda_w1}.
\end{proof}

\subsection{Proof of Theorem~\ref{thm:fm_controls_w2}}
\label{sec:proof_fm_controls_w2}

\begin{proof}
Let $P=P_i^Z$ and $Q=P_{i+1}^Z$ be the representation distributions of two consecutive domains $\mathcal{D}_i$ and $\mathcal{D}_{i+1}$. Let $\pi\in\Gamma(P,Q)$ be a coupling between $P$ and $Q$, where $\Gamma(P,Q)$ denotes the set of all joint distributions whose first marginal is $P$ and second marginal is $Q$. We sample $(z_l,z_r)\sim\pi$, where $z_l\sim P$ is a representation from the current domain and $z_r\sim Q$ is a representation from the next domain.

For $t\in[0,1]$, define the interpolation path
\begin{equation}
z_t=(1-t)z_l+t z_r.
\label{eq:proof_interp_path_direct}
\end{equation}
Let $P_t$ denote the distribution of $z_t$. The conditional velocity along this path is
\begin{equation}
u_t^\star(z_t\mid z_l,z_r)
=
\frac{d z_t}{dt}
=
z_r-z_l.
\label{eq:proof_cond_velocity_direct}
\end{equation}
Since multiple pairs $(z_l,z_r)$ can produce the same interpolated representation $z_t=z$, the velocity field acting on the marginal distribution $P_t$ is the conditional expectation
\begin{equation}
v_t^\star(z)
=
\mathbb{E}_{\pi}
\left[
z_r-z_l
\mid
z_t=z
\right].
\label{eq:proof_marginal_velocity_direct}
\end{equation}
Here $v_t^\star(z)$ is called the marginal velocity field of the interpolation path.

We first show that $(P_t,v_t^\star)$ defines a valid transport path from $P$ to $Q$. Let $f\in C_c^1(\mathcal{Z})$ be any continuously differentiable test function with compact support. By definition of $P_t$,
\begin{equation}
\mathbb{E}_{z\sim P_t}[f(z)]
=
\mathbb{E}_{(z_l,z_r)\sim\pi}
\left[
f((1-t)z_l+t z_r)
\right].
\label{eq:proof_weak_start}
\end{equation}
Taking derivative with respect to $t$ and using the chain rule gives
\begin{equation}
\begin{aligned}
\frac{d}{dt}
\mathbb{E}_{z\sim P_t}[f(z)]
&=
\frac{d}{dt}
\mathbb{E}_{(z_l,z_r)\sim\pi}
\left[
f((1-t)z_l+t z_r)
\right] \\
&=
\mathbb{E}_{(z_l,z_r)\sim\pi}
\left[
\nabla f(z_t)^\top
\frac{d z_t}{dt}
\right] \\
&=
\mathbb{E}_{(z_l,z_r)\sim\pi}
\left[
\nabla f(z_t)^\top
(z_r-z_l)
\right].
\end{aligned}
\label{eq:proof_weak_chain}
\end{equation}
Using the definition of $v_t^\star$ in~\eqref{eq:proof_marginal_velocity_direct}, we can rewrite the last line as
\begin{equation}
\mathbb{E}_{(z_l,z_r)\sim\pi}
\left[
\nabla f(z_t)^\top
(z_r-z_l)
\right]
=
\mathbb{E}_{z\sim P_t}
\left[
\nabla f(z)^\top v_t^\star(z)
\right].
\label{eq:proof_weak_marginal}
\end{equation}
Therefore,
\begin{equation}
\frac{d}{dt}
\mathbb{E}_{z\sim P_t}[f(z)]
=
\mathbb{E}_{z\sim P_t}
\left[
\nabla f(z)^\top v_t^\star(z)
\right].
\label{eq:proof_weak_continuity}
\end{equation}
Equation~\eqref{eq:proof_weak_continuity} is the weak form of the continuity equation
\begin{equation}
\partial_t P_t+\nabla\cdot(P_t v_t^\star)=0,
\label{eq:proof_continuity_direct}
\end{equation}
where $\nabla\cdot$ is the divergence operator with respect to $z$. Thus, $(P_t,v_t^\star)$ is an admissible transport path from $P$ to $Q$ with boundary conditions $P_0=P$ and $P_1=Q$.

By the Benamou--Brenier dynamic formulation of optimal transport, the squared Wasserstein-2 distance satisfies
\begin{equation}
W_2^2(P,Q)
=
\inf_{(P_t,v_t)}
\int_0^1
\mathbb{E}_{z\sim P_t}
\left[
\|v_t(z)\|_2^2
\right]dt,
\label{eq:proof_bb_dynamic}
\end{equation}
where the infimum is over all admissible paths satisfying the continuity equation and the boundary conditions $P_0=P$, $P_1=Q$. Since $(P_t,v_t^\star)$ is one admissible path, we have
\begin{equation}
W_2^2(P,Q)
\le
\int_0^1
\mathbb{E}_{z\sim P_t}
\left[
\|v_t^\star(z)\|_2^2
\right]dt.
\label{eq:proof_w2_energy_direct}
\end{equation}

We now upper-bound the transport energy of $v_t^\star$. By definition,
\begin{equation}
v_t^\star(z)
=
\mathbb{E}_{\pi}
\left[
z_r-z_l
\mid
z_t=z
\right].
\label{eq:proof_vstar_repeat_direct}
\end{equation}
Using Jensen's inequality for the convex function $a\mapsto \|a\|_2^2$, we get
\begin{equation}
\begin{aligned}
\mathbb{E}_{z\sim P_t}
\left[
\|v_t^\star(z)\|_2^2
\right]
&=
\mathbb{E}_{z_t}
\left[
\left\|
\mathbb{E}_{\pi}
\left[
z_r-z_l
\mid
z_t
\right]
\right\|_2^2
\right] \\
&\le
\mathbb{E}_{z_t}
\left[
\mathbb{E}_{\pi}
\left[
\|z_r-z_l\|_2^2
\mid
z_t
\right]
\right] \\
&=
\mathbb{E}_{(z_l,z_r)\sim\pi}
\left[
\|z_r-z_l\|_2^2
\right].
\end{aligned}
\label{eq:proof_jensen_direct}
\end{equation}
Define the coupling cost
\begin{equation}
C_\pi
=
\mathbb{E}_{(z_l,z_r)\sim\pi}
\left[
\|z_r-z_l\|_2^2
\right].
\label{eq:proof_cpi_direct}
\end{equation}
Since~\eqref{eq:proof_jensen_direct} holds for every $t\in[0,1]$, integrating over $t$ gives
\begin{equation}
\int_0^1
\mathbb{E}_{z\sim P_t}
\left[
\|v_t^\star(z)\|_2^2
\right]dt
\le
\int_0^1 C_\pi\,dt
=
C_\pi.
\label{eq:proof_energy_cpi_direct}
\end{equation}
Combining~\eqref{eq:proof_w2_energy_direct} and~\eqref{eq:proof_energy_cpi_direct}, we obtain
\begin{equation}
W_2^2(P,Q)
\le
\int_0^1
\mathbb{E}_{z\sim P_t}
\left[
\|v_t^\star(z)\|_2^2
\right]dt
\le
C_\pi.
\label{eq:proof_w2_final_direct}
\end{equation}

It remains to connect this marginal velocity field to the flow matching objective. The conditional flow matching loss between $P$ and $Q$ is
\begin{equation}
\mathcal{L}_{FM}(u)
=
\mathbb{E}_{t\sim\operatorname{Unif}[0,1]}
\mathbb{E}_{(z_l,z_r)\sim\pi}
\left[
\left\|
u(z_t,t)-(z_r-z_l)
\right\|_2^2
\right],
\label{eq:proof_fm_direct}
\end{equation}
where $u:\mathcal{Z}\times[0,1]\rightarrow\mathcal{Z}$ is any measurable velocity field. Let
\begin{equation}
a=z_r-z_l.
\label{eq:proof_a_direct}
\end{equation}
For fixed $(z,t)$, the conditional expected squared error is
\begin{equation}
\mathbb{E}
\left[
\|u(z,t)-a\|_2^2
\mid
z_t=z
\right].
\label{eq:proof_conditional_error_direct}
\end{equation}
Using
\begin{equation}
v_t^\star(z)
=
\mathbb{E}[a\mid z_t=z],
\label{eq:proof_conditional_mean_direct}
\end{equation}
we decompose
\begin{equation}
\begin{aligned}
\|u(z,t)-a\|_2^2
&=
\|u(z,t)-v_t^\star(z)+v_t^\star(z)-a\|_2^2\\
&=
\|u(z,t)-v_t^\star(z)\|_2^2
+
\|v_t^\star(z)-a\|_2^2\\
&\quad+
2\left\langle
u(z,t)-v_t^\star(z),
v_t^\star(z)-a
\right\rangle.
\end{aligned}
\label{eq:proof_square_decomp_direct}
\end{equation}
Taking conditional expectation given $z_t=z$, the cross term vanishes:
\begin{equation}
\begin{aligned}
&\mathbb{E}
\left[
\left\langle
u(z,t)-v_t^\star(z),
v_t^\star(z)-a
\right\rangle
\mid
z_t=z
\right]\\
&=
\left\langle
u(z,t)-v_t^\star(z),
v_t^\star(z)-\mathbb{E}[a\mid z_t=z]
\right\rangle\\
&=
0.
\end{aligned}
\label{eq:proof_cross_zero_direct}
\end{equation}
Thus,
\begin{equation}
\begin{aligned}
\mathbb{E}
\left[
\|u(z,t)-a\|_2^2
\mid
z_t=z
\right]
&=
\|u(z,t)-v_t^\star(z)\|_2^2\\
&\quad+
\mathbb{E}
\left[
\|v_t^\star(z)-a\|_2^2
\mid
z_t=z
\right].
\end{aligned}
\label{eq:proof_fm_projection_direct}
\end{equation}
The second term in~\eqref{eq:proof_fm_projection_direct} does not depend on $u$. Therefore, the population minimizer of~\eqref{eq:proof_fm_direct} is
\begin{equation}
u^\star(z,t)=v_t^\star(z).
\label{eq:proof_fm_minimizer_direct}
\end{equation}
Consequently, minimizing the flow matching loss learns the marginal velocity field whose kinetic energy upper-bounds $W_2^2(P,Q)$ as shown in~\eqref{eq:proof_w2_final_direct}. This proves Theorem~\ref{thm:fm_controls_w2}.
\end{proof}

\subsection{Proof of Corollary~\ref{cor:flow_to_risk}}
\label{sec:proof_flow_to_risk}

\begin{proof}
For each consecutive domain pair $(\mathcal{D}_i,\mathcal{D}_{i+1})$, let $P_{i,t}$ be the interpolation path between $P_i^Z$ and $P_{i+1}^Z$, and let $v_{i,t}^\star$ be its marginal velocity field. Define the per-step transport energy
\begin{equation}
\mathcal{E}_i
=
\int_0^1
\mathbb{E}_{z\sim P_{i,t}}
\left[
\|v_{i,t}^\star(z)\|_2^2
\right]dt,
\label{eq:proof_energy_i}
\end{equation}
where $P_{i,t}$ is the representation distribution at interpolation time $t$ between the two consecutive domains.

From Theorem~\ref{thm:fm_controls_w2},
\begin{equation}
W_2^2(P_i^Z,P_{i+1}^Z)
\le
\mathcal{E}_i.
\label{eq:proof_cor_w2}
\end{equation}
Taking square roots gives
\begin{equation}
W_2(P_i^Z,P_{i+1}^Z)
\le
\sqrt{\mathcal{E}_i}.
\label{eq:proof_cor_sqrt}
\end{equation}
For probability distributions with finite second moments, we have
\begin{equation}
W_1(P_i^Z,P_{i+1}^Z)
\le
W_2(P_i^Z,P_{i+1}^Z).
\label{eq:proof_w1_le_w2}
\end{equation}
Combining~\eqref{eq:proof_cor_sqrt} and~\eqref{eq:proof_w1_le_w2},
\begin{equation}
W_1(P_i^Z,P_{i+1}^Z)
\le
\sqrt{\mathcal{E}_i}.
\label{eq:proof_cor_w1_energy}
\end{equation}
Substituting~\eqref{eq:proof_cor_w1_energy} into Theorem~\ref{thm:gda_w1}, we obtain
\begin{equation}
\begin{aligned}
\epsilon_T(h)
&\le
\epsilon_0(h)
+
L_\ell
\sum_{i=0}^{T-1}
W_1(P_i^Z,P_{i+1}^Z)
+
\lambda^\star\\
&\le
\epsilon_0(h)
+
L_\ell
\sum_{i=0}^{T-1}
\sqrt{\mathcal{E}_i}
+
\lambda^\star.
\end{aligned}
\label{eq:proof_cor_final}
\end{equation}
This proves Corollary~\ref{cor:flow_to_risk}.
\end{proof}

\subsection{Proof of Lemma~\ref{lem:ce_bound}}
\label{sec:proof_ce_bound}

\begin{proof}
Let $h_s(z)\in[0,1]^K$ be the probability vector predicted by the fixed source classifier $h_s$ at representation $z$. Let
\begin{equation}
p_y(z)=[h_s(z)]_y
\label{eq:proof_py_def}
\end{equation}
be the probability assigned to the true class $y$. The cross-entropy loss is
\begin{equation}
\ell_{\mathrm{CE}}(h_s(z),y)
=
-\log p_y(z).
\label{eq:proof_ce_def}
\end{equation}
Here $[h_s(z)]_y$ denotes the $y$-th component of the probability vector $h_s(z)$.

If $h_s$ misclassifies $z$, then
\begin{equation}
\arg\max_{k\in\{1,\ldots,K\}}[h_s(z)]_k\neq y.
\label{eq:proof_misclassification}
\end{equation}
Therefore, there exists some class $k\neq y$ such that
\begin{equation}
[h_s(z)]_k\ge p_y(z).
\label{eq:proof_competing_class}
\end{equation}
Since all predicted class probabilities sum to one,
\begin{equation}
p_y(z)+[h_s(z)]_k\le 1.
\label{eq:proof_prob_sum}
\end{equation}
Combining~\eqref{eq:proof_competing_class} and~\eqref{eq:proof_prob_sum}, we obtain
\begin{equation}
2p_y(z)\le 1,
\qquad
p_y(z)\le \frac{1}{2}.
\label{eq:proof_py_half}
\end{equation}
Thus, on the misclassification event,
\begin{equation}
\ell_{\mathrm{CE}}(h_s(z),y)
=
-\log p_y(z)
\ge
-\log\left(\frac{1}{2}\right)
=
\log 2.
\label{eq:proof_ce_log2}
\end{equation}
Equation~\eqref{eq:proof_ce_log2} implies the pointwise bound
\begin{equation}
\mathbb{I}
\left[
\arg\max_k [h_s(z)]_k\neq y
\right]
\le
\frac{\ell_{\mathrm{CE}}(h_s(z),y)}{\log 2},
\label{eq:proof_indicator_bound}
\end{equation}
where $\mathbb{I}[\cdot]$ is the indicator function.

Taking expectation over the random pair $(z,y)$ gives
\begin{equation}
\begin{aligned}
\mathbb{P}
\left[
\arg\max_k [h_s(z)]_k\neq y
\right]
&=
\mathbb{E}
\left[
\mathbb{I}
\left[
\arg\max_k [h_s(z)]_k\neq y
\right]
\right]\\
&\le
\frac{1}{\log 2}
\mathbb{E}
\left[
\ell_{\mathrm{CE}}(h_s(z),y)
\right].
\end{aligned}
\label{eq:proof_ce_final}
\end{equation}
This proves Lemma~\ref{lem:ce_bound}.
\end{proof}

\subsection{Proof of Theorem~\ref{thm:label_preserve}}
\label{sec:proof_label_preserve}

\begin{proof}
In FlowGDA, a source representation $z_s\sim P_s^Z$ is transported toward an intermediate domain by the source-to-left flow:
\begin{equation}
\tilde{z}
=
\psi_{s\to i}(z_s),
\label{eq:proof_z_tilde}
\end{equation}
where $\psi_{s\to i}$ is the flow map induced by the velocity field $u_{\theta_{sl}}$, and $\tilde{z}$ is the transported source representation.

The semantic consistency loss is
\begin{equation}
\mathcal{L}_{cls}
=
\mathbb{E}_{(z_s,y)\sim P_s^{Z,Y}}
\left[
\ell_{\mathrm{CE}}(h_s(\tilde{z}),y)
\right],
\label{eq:proof_lcls_def}
\end{equation}
where $P_s^{Z,Y}$ is the source representation-label distribution, $h_s$ is the fixed source classifier, and $y$ is the source label.

Applying Lemma~\ref{lem:ce_bound} to the random pair $(\tilde{z},y)$ yields
\begin{equation}
\begin{aligned}
\mathbb{P}
\left[
\arg\max_k [h_s(\tilde{z})]_k\neq y
\right]
&\le
\frac{1}{\log 2}
\mathbb{E}_{(z_s,y)\sim P_s^{Z,Y}}
\left[
\ell_{\mathrm{CE}}(h_s(\tilde{z}),y)
\right]\\
&=
\frac{\mathcal{L}_{cls}}{\log 2}.
\end{aligned}
\label{eq:proof_label_final}
\end{equation}
This proves Theorem~\ref{thm:label_preserve}.
\end{proof}

\begin{remark}[Why a full JKO derivation is not required]
Some flow and Wasserstein-gradient-flow papers derive a Jordan--Kinderlehrer--Otto (JKO) scheme of the form
\begin{equation}
\rho_{t+\eta}
=
\arg\min_{\rho}
\left[
\frac{1}{2\eta}W_2^2(\rho,\rho_t)
+
\mathcal{F}[\rho]
\right],
\label{eq:proof_jko_remark}
\end{equation}
where $\rho_t$ is a time-indexed density, $\eta>0$ is a step size, and $\mathcal{F}$ is an energy functional. FlowGDA does not optimize a JKO objective. Instead, it uses conditional flow matching to regress velocity fields between neighboring representation distributions. Therefore, the appropriate theoretical tool is the continuity equation together with the Benamou--Brenier dynamic formulation, as used above. This directly matches the algorithmic objective and avoids introducing a variational objective that the method does not optimize.
\end{remark}


\section{Additional Experimental Results}
\label{sec:appendix_additional_exp}

In this appendix, we provide additional analyses that complement the main experimental results. Instead of repeating the main comparison table, we focus on questions that are specific to gradual domain adaptation: how the choice of intermediate domains affects performance, how many intermediate domains are needed, how sensitive FlowGDA is to domain ordering, and how robust it is to key hyperparameters.

\subsection{Effect of Intermediate-Domain Quality}
\label{sec:appendix_intermediate_quality}

We first analyze whether the quality of the intermediate domain affects FlowGDA. On Rotated MNIST, we fix the source domain as unrotated digits ($0^\circ$) and the target domain as digits rotated by $\pi/3$. We then use exactly one intermediate domain $\mathcal{D}_1$ and vary its rotation angle. For each setting, FlowGDA is trained with the same configuration and evaluated on the target domain over five random seeds.

Fig.~\ref{fig:appendix_intermediate_quality} shows that the best performance is obtained when the intermediate domain lies near the midpoint of the source--target trajectory, around $\pi/6$. Intermediate domains that are too close to the source or too close to the target lead to slightly lower accuracy and larger variance. This supports the intuition behind gradual adaptation: the intermediate domain should reduce the local transport difficulty while still lying on the underlying geometric trajectory of the shift. Direct source-to-target adaptation without an intermediate domain performs worse, indicating that the existence of an intermediate bridge is more important than the exact optimality of its location.

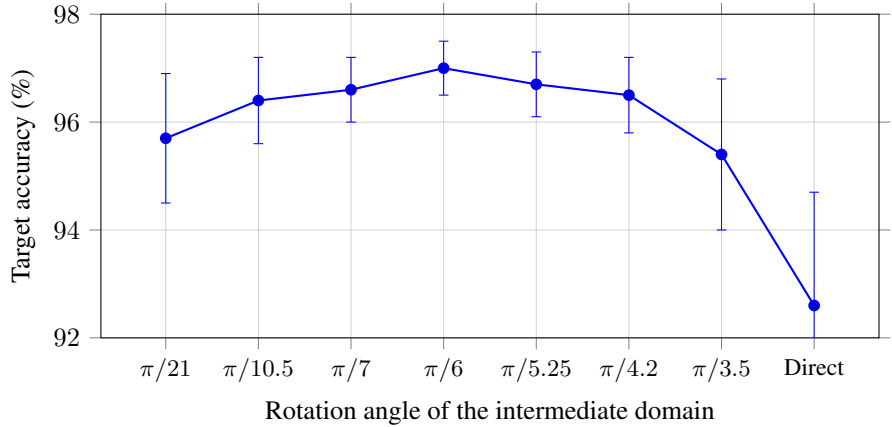
\begin{figure}[t]
\centering
\begin{tikzpicture}
\begin{axis}[
  width=0.85\textwidth,
  height=0.42\textwidth,
  ymin=92, ymax=98,
  grid=both,
  grid style={line width=0.1pt, draw=black!12},
  major grid style={line width=0.2pt, draw=black!18},
  tick align=outside,
  xlabel={Rotation angle of the intermediate domain},
  ylabel={Target accuracy (\%)},
  xtick={1,2,3,4,5,6,7,8},
  xticklabels={$\pi/21$,$\pi/10.5$,$\pi/7$,$\pi/6$,$\pi/5.25$,$\pi/4.2$,$\pi/3.5$,Direct},
  x tick label style={font=\small},
]
\addplot+[
  thick,
  mark=*,
  mark size=1.8pt,
  error bars/.cd,
    y dir=both,
    y explicit,
    error bar style={line width=0.4pt},
] coordinates {
(1,95.7) +- (0,1.2)
(2,96.4) +- (0,0.8)
(3,96.6) +- (0,0.6)
(4,97.0) +- (0,0.5)
(5,96.7) +- (0,0.6)
(6,96.5) +- (0,0.7)
(7,95.4) +- (0,1.4)
(8,92.6) +- (0,2.1)
};
\end{axis}
\end{tikzpicture}
\vspace{-0.3em}
\caption{Effect of intermediate-domain quality on Rotated MNIST. Accuracy is reported as mean$\pm$std over five runs.}
\label{fig:appendix_intermediate_quality}
\end{figure}

\subsection{Effect of the Number of Intermediate Domains}
\label{sec:appendix_num_intermediate}

We study how the number of intermediate domains affects FlowGDA. Table~\ref{tab:appendix_num_intermediate} reports target accuracy when using different numbers of intermediate domains. The setting with zero intermediate domains corresponds to direct source-to-target adaptation. As the number of intermediate domains increases, performance improves consistently across all datasets, with larger gains on Portraits and Cover Type. This is consistent with our theoretical analysis in Section~\ref{sec:theory}: more intermediate domains reduce the discrepancy between consecutive domains, making each local transport problem easier. The improvement gradually saturates when the trajectory becomes sufficiently fine-grained.

\begin{table}[t]
\centering
\caption{Effect of the number of intermediate domains. Values denote target-domain accuracy (\%).}
\label{tab:appendix_num_intermediate}
\renewcommand{\arraystretch}{1.06}
\resizebox{0.96\textwidth}{!}{%
\begin{tabular}{lccccc}
\toprule
\textbf{Dataset} 
& \textbf{0 domains} 
& \textbf{1 domain} 
& \textbf{2 domains} 
& \textbf{4 domains} 
& \textbf{6 domains} \\
\midrule
Rotated MNIST 
& 92.6 
& 95.1 
& 96.5 
& 96.8 
& \textbf{97.0} \\
Color-Shift MNIST 
& 94.7 
& 97.9 
& 99.1 
& 99.2 
& \textbf{99.2} \\
Portraits 
& 80.9 
& 83.2 
& 85.1 
& 86.4 
& \textbf{87.6} \\
Cover Type 
& 71.6 
& 69.2 
& 74.7 
& 75.5 
& \textbf{76.1} \\
\bottomrule
\end{tabular}%
}
\end{table}

The results show two important trends. First, direct adaptation is consistently weaker than gradual adaptation, confirming the need for intermediate domains under large shifts. Second, the marginal gain from adding more domains decreases after a moderate number of intermediate domains. This suggests that FlowGDA benefits from finer trajectories, but does not require arbitrarily dense domain sequences once local shifts become sufficiently small.

\subsection{Robustness to Domain Ordering}
\label{sec:appendix_ordering}

Gradual domain adaptation assumes that domains are ordered so that consecutive domains are closer than the source and target. To evaluate the importance of this assumption, we compare four orderings: the original ordered trajectory, a local shuffle that only swaps neighboring domains, a random permutation, and a reversed order. Table~\ref{tab:appendix_ordering} reports the results.

\begin{table}[t]
\centering
\caption{Robustness to domain ordering. Values denote target-domain accuracy (\%).}
\label{tab:appendix_ordering}
\renewcommand{\arraystretch}{1.06}
\resizebox{0.96\textwidth}{!}{%
\begin{tabular}{lcccc}
\toprule
\textbf{Ordering} 
& \textbf{Rotated MNIST} 
& \textbf{Color-Shift MNIST} 
& \textbf{Portraits} 
& \textbf{Cover Type} \\
\midrule
Ordered trajectory 
& \textbf{97.0} 
& \textbf{99.2} 
& \textbf{87.6} 
& \textbf{76.1} \\
Local shuffle 
& 95.6 
& 98.0 
& 85.4 
& 74.0 \\
Random permutation 
& 92.8 
& 94.7 
& 81.7 
& 70.5 \\
Reversed order 
& 90.7 
& 92.6 
& 79.8 
& 69.2 \\
\bottomrule
\end{tabular}%
}
\end{table}

FlowGDA is robust to mild local perturbations, but performance drops when the domain order is random or reversed. This confirms that the method benefits from a meaningful gradual trajectory. The drop is most pronounced on Portraits and Cover Type, where the shifts are less controlled and the representation trajectory is more complex. These results also suggest that when domain order is noisy in practice, estimating a reasonable ordering or neighborhood graph before adaptation may further improve performance.

\subsection{Impact of Early Transport Steps}
\label{sec:appendix_early_steps}

We further investigate whether early transport steps \cite{yao2026revascenecentricdatasetrepetition, fan2026communicationawaremultiagentreinforcementlearning,huang2026tcn,xiao2025joint} are more important than later ones. We divide the ordered domain chain into three blocks: early, middle, and late domain pairs. We then selectively disable velocity-field training in one or more blocks. When a block is disabled, FlowGDA uses identity transport for that block and only trains on the remaining blocks. Table~\ref{tab:appendix_stage_importance} reports the results on Portraits and Cover Type, the two datasets with the most complex non-synthetic shifts.

\begin{table}[t]
\centering
\caption{Impact of early, middle, and late transport steps. ``Train'' means velocity fields are optimized for the corresponding block, while ``Skip'' means identity transport is used. Values denote target-domain accuracy (\%).}
\label{tab:appendix_stage_importance}
\renewcommand{\arraystretch}{1.06}
\resizebox{0.96\textwidth}{!}{%
\begin{tabular}{lccccc}
\toprule
\textbf{Schedule} 
& \textbf{Early pairs} 
& \textbf{Middle pairs} 
& \textbf{Late pairs} 
& \textbf{Portraits} 
& \textbf{Cover Type} \\
\midrule
Full FlowGDA 
& Train & Train & Train 
& \textbf{87.6} 
& \textbf{76.1} \\
Skip early pairs 
& Skip & Train & Train 
& 83.1 
& 71.8 \\
Skip middle pairs 
& Train & Skip & Train 
& 85.2 
& 73.6 \\
Skip late pairs 
& Train & Train & Skip 
& 86.1 
& 74.9 \\
Only early pairs 
& Train & Skip & Skip 
& 84.7 
& 72.6 \\
Only late pairs 
& Skip & Skip & Train 
& 81.4 
& 70.7 \\
\bottomrule
\end{tabular}%
}
\end{table}

The results show that early transport steps have a larger impact than later ones. Skipping early pairs causes the largest degradation among the single-block removals, while skipping late pairs is less harmful. Moreover, training only the early pairs outperforms training only the late pairs. This indicates that early transport steps establish a stable representation trajectory, while later steps mainly refine the already-aligned features. This behavior is consistent with the intuition that errors introduced at the beginning of a gradual path can propagate and affect subsequent adaptation.

\subsection{Hyperparameter Sensitivity}
\label{sec:appendix_hyperparameter}

Finally, we study the sensitivity of FlowGDA to key hyperparameters on the Portraits dataset. Table~\ref{tab:appendix_hyperparameter} reports the effects of the semantic consistency weight $\lambda_{cls}$, the hidden dimension of the velocity network, the number of curriculum levels for $p$, and batch size. FlowGDA is stable across reasonable ranges of these hyperparameters. The best performance is obtained with $\lambda_{cls}=0.1$, hidden dimension 512, four curriculum levels, and batch size 128.

\begin{table}[t]
\centering
\caption{Hyperparameter sensitivity on Portraits. Values denote target-domain accuracy (\%).}
\label{tab:appendix_hyperparameter}
\begin{tabular}{llc}
\toprule
\textbf{Factor} & \textbf{Setting} & \textbf{Accuracy} \\
\midrule
\multirow{6}{*}{$\lambda_{cls}$}
& 0 & 84.8 \\
& $10^{-2}$ & 86.3 \\
& 0.05 & 87.2 \\
& 0.1 & \textbf{87.6} \\
& 0.2 & 87.1 \\
& 0.5 & 86.0 \\
\midrule
\multirow{4}{*}{Hidden dimension}
& 128 & 85.8 \\
& 256 & 86.9 \\
& 512 & \textbf{87.6} \\
& 1024 & 87.4 \\
\midrule
\multirow{5}{*}{Curriculum levels}
& 1 & 85.9 \\
& 2 & 86.8 \\
& 4 & \textbf{87.6} \\
& 6 & 87.5 \\
& 8 & 87.2 \\
\midrule
\multirow{4}{*}{Batch size}
& 32 & 86.2 \\
& 64 & 87.0 \\
& 128 & \textbf{87.6} \\
& 256 & 87.5 \\
\bottomrule
\end{tabular}%
\end{table}

The semantic consistency weight shows a non-monotonic trend. Without $\mathcal{L}_{cls}$, performance drops because transported representations may lose label information. However, an overly large $\lambda_{cls}$ also hurts performance because it over-constrains the transport and limits adaptation to the target domain. The velocity-network width shows saturation after 512 hidden units, suggesting that excessive capacity is unnecessary. For the curriculum schedule, a moderate number of levels works best: too few levels make adaptation abrupt, while too many levels add optimization overhead without clear benefit. Larger batch sizes generally improve stability, likely because they provide more reliable estimates of local transport directions.

%% file: references.bib
@article{bengio2013representation,
  title={Representation learning: A review and new perspectives},
  author={Bengio, Yoshua and Courville, Aaron and Vincent, Pascal},
  journal={IEEE Transactions on Pattern Analysis and Machine Intelligence},
  volume={35},
  number={8},
  pages={1798--1828},
  year={2013},
  publisher={IEEE}
}

@article{bendavid2010theory,
  title={A theory of learning from different domains},
  author={Ben-David, Shai and Blitzer, John and Crammer, Koby and Kulesza, Alex and Pereira, Fernando and Vaughan, Jennifer Wortman},
  journal={Machine Learning},
  volume={79},
  number={1},
  pages={151--175},
  year={2010},
  publisher={Springer}
}

@inproceedings{chen2018neural,
  title={Neural ordinary differential equations},
  author={Chen, Ricky TQ and Rubanova, Yulia and Bettencourt, Jesse and Duvenaud, David K},
  booktitle={Advances in Neural Information Processing Systems},
  pages={6571--6583},
  year={2018}
}

@article{courty2015optimal,
  title={Optimal transport for domain adaptation},
  author={Courty, Nicolas and Flamary, R{\'e}mi and Tuia, Devis and Rakotomamonjy, Alain},
  journal={IEEE Transactions on Pattern Analysis and Machine Intelligence},
  volume={39},
  number={9},
  pages={1853--1865},
  year={2017},
  publisher={IEEE}
}

@article{ganin2016domain,
  title={Domain-adversarial training of neural networks},
  author={Ganin, Yaroslav and Ustinova, Evgeniya and Ajakan, Hana and Germain, Pascal and Larochelle, Hugo and Laviolette, Fran{\c{c}}ois and Marchand, Mario and Lempitsky, Victor},
  journal={Journal of Machine Learning Research},
  volume={17},
  number={1},
  pages={2096--2030},
  year={2016}
}

@inproceedings{goodfellow2014generative,
  title={Generative adversarial nets},
  author={Goodfellow, Ian and Pouget-Abadie, Jean and Mirza, Mehdi and Xu, Bing and Warde-Farley, David and Ozair, Sherjil and Courville, Aaron and Bengio, Yoshua},
  booktitle={Advances in Neural Information Processing Systems},
  pages={2672--2680},
  year={2014}
}

@inproceedings{gulrajani2017improved,
  title={Improved training of Wasserstein GANs},
  author={Gulrajani, Ishaan and Ahmed, Faruk and Arjovsky, Martin and Dumoulin, Vincent and Courville, Aaron C},
  booktitle={Advances in Neural Information Processing Systems},
  pages={5767--5777},
  year={2017}
}

@inproceedings{lipman2023flow,
  title={Flow matching for generative modeling},
  author={Lipman, Yaron and Chen, Ricky TQ and Ben-Hamu, Heli and Nickel, Maximilian and Le, Matthew},
  booktitle={International Conference on Learning Representations},
  year={2023}
}

@inproceedings{long2015learning,
  title={Learning transferable features with deep adaptation networks},
  author={Long, Mingsheng and Cao, Yue and Wang, Jianmin and Jordan, Michael I},
  booktitle={International Conference on Machine Learning},
  pages={97--105},
  year={2015},
  organization={PMLR}
}

@article{pan2010survey,
  title={A survey on transfer learning},
  author={Pan, Sinno Jialin and Yang, Qiang},
  journal={IEEE Transactions on Knowledge and Data Engineering},
  volume={22},
  number={10},
  pages={1345--1359},
  year={2010},
  publisher={IEEE}
}

@article{peyré2019computational,
  title={Computational optimal transport},
  author={Peyr{\'e}, Gabriel and Cuturi, Marco},
  journal={Foundations and Trends in Machine Learning},
  volume={11},
  number={5-6},
  pages={355--607},
  year={2019},
  publisher={Now Publishers, Inc.}
}

@inproceedings{rezende2015variational,
  title={Variational inference with normalizing flows},
  author={Rezende, Danilo Jimenez and Mohamed, Shakir},
  booktitle={International Conference on Machine Learning},
  pages={1530--1538},
  year={2015},
  organization={PMLR}
}

@inproceedings{sun2016deep,
  title={Deep CORAL: Correlation alignment for deep domain adaptation},
  author={Sun, Baochen and Saenko, Kate},
  booktitle={European Conference on Computer Vision Workshops},
  pages={443--450},
  year={2016},
  organization={Springer}
}

@inproceedings{tzeng2017adversarial,
  title={Adversarial discriminative domain adaptation},
  author={Tzeng, Eric and Hoffman, Judy and Saenko, Kate and Darrell, Trevor},
  booktitle={Proceedings of the IEEE Conference on Computer Vision and Pattern Recognition},
  pages={7167--7176},
  year={2017}
}

@book{villani2009optimal,
  title={Optimal Transport: Old and New},
  author={Villani, C{\'e}dric},
  year={2009},
  publisher={Springer Science \& Business Media}
}

@inproceedings{wang2022understanding,
  title={Understanding gradual domain adaptation: Improved analysis, optimal path and beyond},
  author={Wang, Haoxiang and Li, Bo and Zhao, Han},
  booktitle={International Conference on Machine Learning},
  pages={22784--22801},
  year={2022},
  organization={PMLR}
}

@inproceedings{wdgrl2018,
  title={Wasserstein distance guided representation learning for domain adaptation},
  author={Shen, Jian and Qu, Yanru and Zhang, Weinan and Yu, Yong},
  booktitle={AAAI Conference on Artificial Intelligence},
  pages={4058--4065},
  year={2018}
}

@inproceedings{zhang2019bridging,
  title={Bridging theory and algorithm for domain adaptation},
  author={Zhang, Yuchen and Liu, Tianle and Long, Mingsheng and Jordan, Michael I},
  booktitle={International Conference on Machine Learning},
  pages={7404--7413},
  year={2019},
  organization={PMLR}
}

@article{yao2022riemannian,
  title={Riemannian flow matching on general geometries. 2023},
  author={Chen, RT and Lipman, Yaron},
  journal={ArXiv preprint ArXiv230203660}
}

@inproceedings{arjovsky2017wasserstein,
  title={Wasserstein generative adversarial networks},
  author={Arjovsky, Martin and Chintala, Soumith and Bottou, L{\'e}on},
  booktitle={International Conference on Machine Learning},
  pages={214--223},
  year={2017},
  organization={PMLR}
}

@inproceedings{loshchilov2019decoupled,
  title={Decoupled weight decay regularization},
  author={Loshchilov, Ilya and Hutter, Frank},
  booktitle={International Conference on Learning Representations},
  year={2019}
}

@article{pooladian2023categorical,
  title={Categorical flow matching on statistical manifolds},
  author={Cheng, Chaoran and Li, Jiahan and Peng, Jian and Liu, Ge},
  journal={Advances in Neural Information Processing Systems},
  volume={37},
  pages={54787--54819},
  year={2024}
}

@inproceedings{chen2023riemannian,
  title={Riemannian flow matching on general geometries},
  author={Chen, Ricky TQ and Lipman, Yaron},
  booktitle={International Conference on Learning Representations},
  year={2023}
}

@inproceedings{li2025gradual,
  title={Gradual domain adaptation via manifold-constrained distributionally robust optimization},
  author={Najafi, Amir and Behjati, Amin and Emrani, Ala and Zolfimoselo, Yasaman and Shadrooy, Shadrooy and Motahari, Abolfazl and Khalaj, Babak and others},
  journal={Advances in Neural Information Processing Systems},
  volume={37},
  pages={73693--73725},
  year={2024}
}

@article{wang2025gradual,
  title={Gradual domain adaptation for graph learning},
  author={Lei, Pui Ieng and Chen, Ximing and Sheng, Yijun and Liu, Yanyan and Gong, Zhiguo and Yang, Qiang},
  journal={arXiv preprint arXiv:2501.17443},
  year={2025}
}

@article{flamary2024recent,
  title={Recent advances in optimal transport for machine learning},
  author={Montesuma, Eduardo Fernandes and Mboula, Fred Maurice Ngole and Souloumiac, Antoine},
  journal={IEEE Transactions on Pattern Analysis and Machine Intelligence},
  year={2024},
  publisher={IEEE}
}

@inproceedings{zhao2025universal,
  title={Universal Neural Optimal Transport},
  author={Geuter, Jonathan and Kornhardt, Gregor and Tomasson, Ingimar and Laschos, Vaios},
  booktitle={Forty-second International Conference on Machine Learning},
  year={2025}
}

@inproceedings{kumar2020understanding,
  title={Understanding self-training for gradual domain adaptation},
  author={Kumar, Ananya and Ma, Tengyu and Liang, Percy},
  booktitle={International conference on machine learning},
  pages={5468--5479},
  year={2020},
  organization={PMLR}
}

@article{shi2024adversarial,
  title={Adversarial self-training improves robustness and generalization for gradual domain adaptation},
  author={Shi, Lianghe and Liu, Weiwei},
  journal={Advances in Neural Information Processing Systems},
  volume={36},
  year={2024}
}

@article{he2023gradual,
  title={Gradual Domain Adaptation: Theory and Algorithms},
  author={He, Yifei and Wang, Haoxiang and Li, Bo and Zhao, Han},
  journal={arXiv preprint arXiv:2310.13852},
  year={2023}
}

@article{sagawa2025gradual,
  title={Gradual domain adaptation via normalizing flows},
  author={Sagawa, Shogo and Hino, Hideitsu},
  journal={Neural Computation},
  pages={1--47},
  year={2025},
  publisher={MIT Press 255 Main Street, 9th Floor, Cambridge, Massachusetts 02142, USA~…}
}

@article{long2018conditional,
  title={Conditional adversarial domain adaptation},
  author={Long, Mingsheng and Cao, Zhangjie and Wang, Jianmin and Jordan, Michael I},
  journal={Advances in neural information processing systems},
  volume={31},
  year={2018}
}

@article{yao2026disa,
  title={DiSa: Saliency-Aware Foreground-Background Disentangled Framework for Open-Vocabulary Semantic Segmentation},
  author={Yao, Zhen and Li, Xin and Jing, Taotao and Zhang, Shuai and Chuah, Mooi Choo},
  journal={arXiv preprint arXiv:2601.20064},
  year={2026}
}

@article{yao2025learning,
  title={Learning Flow-Guided Registration for RGB-Event Semantic Segmentation},
  author={Yao, Zhen and Ying, Xiaowen and Zhu, Zhiyu and Chuah, Mooi Choo},
  journal={arXiv preprint arXiv:2505.01548},
  year={2025}
}

@misc{yao2026revascenecentricdatasetrepetition,
      title={ReVA: A Scene-Centric Dataset Beyond Repetition for Remote Sensing Video Question Answering}, 
      author={Zhen Yao and Likai Wang and Yuming Yang and Zhihao Zheng and Bo Lang and Qiuyu Tang and Jialu Sheng and Jingqi Xu and Yuehai Yang and Jumal Barker and Xiaowen Ying and Mooi Choo Chuah},
      year={2026},
      eprint={2609.35507},
      archivePrefix={arXiv},
      primaryClass={cs.CV},
      url={https://arxiv.org/abs/2609.35507}, 
}

@misc{fan2026communicationawaremultiagentreinforcementlearning,
      title={Communication-Aware Multi-Agent Reinforcement Learning for Decentralized Cooperative UAV Deployment}, 
      author={Enguang Fan and Yifan Chen and Zihan Shan and Klara Nahrstedt and Matthew Caesar and Jae Kim},
      year={2026},
      eprint={2603.16141},
      archivePrefix={arXiv},
      primaryClass={cs.MA},
      url={https://arxiv.org/abs/2603.16141}, 
}

@inproceedings{wang2025multi,
  title={Multi-modal gradual domain osmosis: Stepwise dynamic learning with batch matching for gradual domain adaptation},
  author={Wang, Zixi and Huang, Yubo and Xu, Jingzehua and Wei, Jinzhu and Zhang, Shuai and Lai, Xin},
  booktitle={Proceedings of the 33rd ACM International Conference on Multimedia},
  pages={8959--8967},
  year={2025}
}

@article{wang2026self,
  title={Self-Training with Dynamic Weighting for Robust Gradual Domain Adaptation},
  author={Wang, Zixi and Cao, Yushe and Huang, Yubo and Wei, Jinzhu and Xu, Jingzehua and Zhang, Shuai and Lai, Xin},
  journal={Advances in Neural Information Processing Systems},
  volume={38},
  pages={107257--107278},
  year={2026}
}

@article{huang2026tcn,
  title={TCN-AOL: An online incremental learning method for real-time multi-step prediction of shield tunnel excavation performance},
  author={Huang, Yubo and Hu, Xiongyu and Lai, Xin and Cheng, Pengyue and Liu, Sijin and Fang, Yong and He, Chuan and Wang, Zixi and Xu, Jingzehua},
  journal={Computer-Aided Civil and Infrastructure Engineering},
  pages={100138},
  year={2026},
  publisher={Elsevier}
}

@inproceedings{xiao2026yang,
  title={A YANG-Grounded LLM Agent Supporting Multi-Vendor OpenROADM Optical Transport Networks},
  author={Xiao, Linqi and Latha, Aparaajitha Gomathinayakam and Gudepu, Venkateswarlu and Fumagalli, Andrea},
  booktitle={Proceedings of the ACM SIGCOMM 2026 Conference},
  pages={2109--2115},
  year={2026}
}

@inproceedings{xiao2025joint,
  title={Joint Alarm Correlation and Root Alarm Detection in Optical Transport Networks with Graph-Attention},
  author={Xiao, Linqi and Gudepu, Venkateswarlu and Fumagalli, Andrea},
  booktitle={2025 IEEE 30th International Workshop on Computer Aided Modeling and Design of Communication Links and Networks (CAMAD)},
  pages={1--6},
  year={2025},
  organization={IEEE}
}

@article{yang2025automated,
  title={Automated classification of microstructure data in additive manufacturing via low-dimensional learning},
  author={Yang, Wei and Grasso, Marco and Bisheh, Mohammad N and Paynabar, Kamran and Colosimo, Bianca Maria},
  journal={Additive Manufacturing},
  pages={104909},
  year={2025},
  publisher={Elsevier}
}
